\documentclass{article}

\usepackage[nonatbib,preprint]{neurips_2020}

\usepackage[utf8]{inputenc} 
\usepackage[T1]{fontenc}    
\usepackage{url}            
\usepackage{booktabs}       
\usepackage{amsfonts}       
\usepackage{nicefrac}       
\usepackage{microtype}      

\usepackage{graphicx}

\usepackage{comment}
\usepackage{amsmath,amssymb, amsthm} 

\newtheorem*{remark}{Remark}
\newtheorem{proposition}{Proposition}[section]

\newcommand*\samethanks[1][\value{footnote}]{\footnotemark[#1]}
\usepackage{xcolor}
\usepackage{pgfplots}

\usepackage{booktabs}
\usepackage{multirow}
\pgfplotsset{compat=1.14}
\usepackage{array}

\newcommand{\mm}[1]{{\color{red}MICHAEL---#1---}}

\newcommand*{\R}{\mathbb{R}}
\newcommand*{\cl}{\mathcal{L}}
\newcommand*{\lbr}{\left(}

\newcommand*{\rbr}{\right)}
\newcommand*\diff{\mathop{}\!\mathrm{d}}
\usepackage{dirtytalk}
\usepackage{mathtools}

\usepackage{svg}
\usepackage{wrapfig}
\usepackage{array}
\newcolumntype{P}[1]{>{\centering\arraybackslash}p{#1}}

\title{Inverting Gradients - How easy is it to break privacy in federated learning?}

  \author{
 Jonas Geiping\thanks{Authors contributed equally.} 
 \And Hartmut Bauermeister \samethanks \And Hannah Dröge \samethanks \And
 Michael Moeller \\
\phantom{user}\\
 Dep. of Electrical Engineering and Computer Science\\
  University of Siegen\\
\texttt{\{jonas.geiping, hartmut.bauermeister, hannah.droege,}\\
\texttt{michael.moeller \}@uni-siegen.de}
}
\begin{document}
\maketitle
\vspace{-0.5cm}
\begin{abstract}
The idea of federated learning is to collaboratively train a neural network on a server. Each user receives the current weights of the network and in turns sends parameter updates (gradients) based on local data. This protocol has been designed not only to train neural networks data-efficiently, but also to provide privacy benefits for users, as their input data remains on device and only parameter gradients are shared. 
But how secure is sharing parameter gradients? Previous attacks have provided a false sense of security, by succeeding only in contrived settings - even for a single image. However, by exploiting a magnitude-invariant loss along with optimization strategies based on adversarial attacks, we show that is is actually possible to faithfully reconstruct images at high resolution from the knowledge of their parameter gradients, and demonstrate that such a break of privacy is possible even for trained deep networks.
We analyze the effects of architecture as well as parameters on the difficulty of reconstructing an input image and prove that any input to a fully connected layer can be reconstructed analytically independent of the remaining architecture. Finally we discuss settings encountered in practice and show that even averaging gradients over several iterations or several images does not protect the user's privacy in federated learning applications.
\end{abstract}

\section{Introduction}
Federated or collaborative learning \cite{chilimbi_project_2014,shokri_privacy-preserving_2015} is a distributed learning paradigm that has recently gained significant attention as both data requirements and privacy concerns in machine learning continue to rise \cite{mcmahan_communication-efficient_2017,jochems_distributed_2016,yang_federated_2019}. The basic idea is to train a machine learning model, for example a neural network, by optimizing the parameters $\theta$ of the network using a loss function $\mathcal{L}$ and exemplary training data consisting of input images $x_i$ and corresponding labels $y_i$ in order to solve
\begin{equation}
\label{eq:training}
     \min_\theta \sum_{i=1}^N \mathcal{L}_\theta(x_i, y_i).    
\end{equation}
\begin{figure}[thb]
    \centering
    \includegraphics[width=0.32\textwidth]{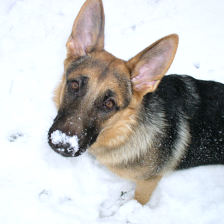}
    \includegraphics[width=0.32\textwidth]{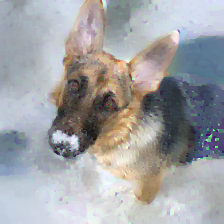}
    \includegraphics[width=0.32\textwidth]{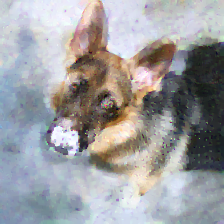}
    \caption{Reconstruction of an input image $x$ from the gradient $\nabla_{\theta} \mathcal{L}_{\theta}(x, y)$. Left: Image from the validation dataset. Middle: Reconstruction from a trained ResNet-18 trained on ImageNet. Right: Reconstruction from a trained ResNet-152. In both cases, the intended privacy of the image is broken. Also note that previous attacks cannot recover ImageNet-sized data \cite{zhu_deep_2019}.}
    \label{fig:teaser}
\end{figure}
We consider a distributed setting in which a \textit{server} wants to solve \eqref{eq:training} with the help of multiple \textit{users} that own training data $(x_i,y_i)$. The idea of federated learning is to only share the gradients $\nabla_{\theta} \mathcal{L}_{\theta}(x_i, y_i)$ instead of the original data $(x_i,y_i)$ with the server which it subsequently accumulates to update the overall weights. Using gradient descent the server's updates could, for instance, constitute 
\begin{equation}
\label{eq:fedSGD}
    \theta^{k+1} = \underbrace{\theta^k - \tau \sum_{i=1}^N }_{\text{server}} \underbrace{\vphantom{\sum_{i=1}^N}\nabla_{\theta} \mathcal{L}_{\theta^k}(x_i, y_i)}_{\text{users}}. 
\end{equation}
The updated parameters $\theta^{k+1}$ are sent back to the individual users. The procedure in eq.~\eqref{eq:fedSGD} is called \textit{federated SGD}. In contrast, in \textit{federated averaging} \cite{konecny_federated_2015,mcmahan_communication-efficient_2017} each user computes several gradient descent steps locally, and sends the updated parameters back to the server. Finally, information about $(x_i,y_i)$ can be further obscured, by only sharing the mean $ \frac{1}{t} \sum_{i=1}^t\nabla_{\theta} \mathcal{L}_{\theta^k}(x_i, y_i)$ of the gradients of several local examples, which we refer to as the \textit{multi-image} setting. 

Distributed learning of this kind has been used in real-world applications where user privacy is crucial, e.g. for hospital data \cite{jochems_developing_2017} or text predictions on mobile devices \cite{bonawitz_towards_2019}, and it has been stated that ``Privacy is enhanced by the ephemeral and focused nature of the [Federated Learning] updates'' \cite{bonawitz_towards_2019}: model updates are considered to contain less information than the original data, and through aggregation of updates from multiple data points, original data is considered impossible to recover. In this work we show analytically as well as empirically, that parameter gradients still carry significant information about the supposedly private input data as we illustrate in Fig.~\ref{fig:teaser}. We conclude by showing that even \textit{multi-image federated averaging} on realistic architectures does not guarantee the privacy of all user data, showing that out of a batch of 100 images, several are still recoverable.

\textbf{Threat model:} We investigate an \textit{honest-but-curious} server with the goal of uncovering user data: The attacker is allowed to separately store and process updates transmitted by individual users, but may \textit{not} interfere with the collaborative learning algorithm. The attacker may not modify the model architecture to better suit their attack, nor send malicious global parameters that do not represent the actually learned global model. The user is allowed to accumulate data locally in Sec.~\ref{sec:realistic}. We refer to the supp. material for further commentary and mention that the attack is near-trivial under weaker constraints on the attacker.



In this paper we discuss privacy limitations of federated learning first in an academic setting, honing in on the case of gradient inversion from one image and showing that
\begin{itemize}
    \item Reconstruction of input data from gradient information is possible for realistic deep and non-smooth architectures with both, trained and untrained parameters.
    \item With the right attack, there is little ``defense-in-depth" - deep networks are as vulnerable as shallow networks. 
    \item We prove that the input to any fully connected layer can be reconstructed analytically independent of the remaining network architecture.
\end{itemize}
Then we consider the implications that the findings have for practical scenarios, finding that
\begin{itemize}
    \item Reconstruction of multiple, separate input images from their averaged gradient is possible in practice, over multiple epochs, using local mini-batches, or even for a local gradient averaging of up to 100 images.
\end{itemize}

\section{Related Work}
Previous related works that investigate recovery from gradient information have been limited to shallow networks of less practical relevance. 
Recovery of image data from gradient information was first discussed in \cite{phong_privacy-preserving_2017,phong_privacy-preserving_2017-1} for neural networks, who prove that recovery is possible for a single neuron or linear layer. For convolutional architectures, \cite{wang_beyond_2018} show that recovery of a single image is possible for a 4-layer CNN, albeit with a significantly large fully-connected (FC) layer. Their work first constructs a ``representation" of the input image, that is then improved with a GAN. 
\cite{zhu_deep_2019} extends this, showing for a 4-layer CNN (with a large FC layer, smooth sigmoid activations, no strides, uniformly random weights), that missing label information 
can also be jointly reconstructed. They further show that reconstruction of multiple images from their averaged gradients is indeed possible (for a maximum batch size of 8). \cite{zhu_deep_2019} also discuss deeper architectures, but provide no tangible results. A follow-up \cite{zhao_idlg_2020} notes that label information can be computed analytically from the gradients of the last layer.
These works make strong assumptions on the model architecture and model parameters that make reconstructions easier, but violate the threat model that we consider in this work and lead to less realistic scenarios.

The central recovery mechanism discussed in \cite{wang_beyond_2018,zhu_deep_2019,zhao_idlg_2020} is the optimization of an euclidean matching term.
The cost function 
\begin{equation}\label{eq:l2_penalty}
    \arg \min_x ||\nabla_\theta \mathcal{L}_\theta(x, y) - \nabla_\theta \mathcal{L}_\theta(x^*, y)||^2
\end{equation}
is minimized to recover the original input image $x^*$ from a transmitted gradient $\nabla_\theta \mathcal{L}_\theta(x^*, y)$. This optimization problem is solved by an L-BFGS solver \cite{liu_limited_1989}. Note that differentiating the gradient of $\mathcal{L}$ w.r.t to $x$ requires a second-order derivative of the considered parametrized function and L-BFGS needs to construct a third-order derivative approximation, which is challenging for neural networks with ReLU units for which higher-order derivatives are discontinuous.

A related, but easier problem, compared to the full reconstruction of input images, is the retrieval of input attributes \cite{melis_exploiting_2019,ganju_property_2018} from local updates, e.g. does a person that is recognized in a face recognition system wear a hat. Information even about attributes unrelated to the task at-hand can be recovered from deeper layers of a neural network, which can be recovered from local updates.

Our problem statement is furthermore related to model inversion \cite{fredrikson_model_2015}, where training images are recovered from network parameters after training. This provides a natural limit case for our setting. Model inversion generally is challenging for deeper neural network architectures \cite{zhang_secret_2019} if no additional information is given \cite{fredrikson_model_2015,zhang_secret_2019}.
Another closely related task is inversion from visual representations \cite{dosovitskiy_inverting_2016,dosovitskiy_generating_2016,mahendran_visualizing_2016}, where, given the output of some intermediate layer of a neural network, a plausible input image is reconstructed. This procedure can leak some information, e.g. general image composition, dominating colors - but, depending on the given layer it only reconstructs similar images - if the neural network is not explicitly chosen to be (mostly) invertible \cite{jacobsen_i-revnet_2018}. As we prove later, inversion from visual representations is strictly more difficult than recovery from gradient information.

\section{Theoretical Analysis: Recovering Images from their Gradients}
\label{chapter:fc_layers}
To understand the overall problem of breaking privacy in federated learning from a theoretical perspective, let us first analyze the question if data $x \in \R^n$ can be recovered from its gradient $\nabla_\theta \mathcal{L}_\theta(x,y) \in \R^p$ analytically.

Due to the different dimensionality of $x$ and $\nabla_\theta \mathcal{L}_\theta(x,y)$, reconstruction quality is surely is a question of the number of parameters $p$ versus input pixels $n$. If $p < n$, then reconstruction is at least as difficult as image recovery from incomplete data \cite{candes_robust_2006,benning_modern_2018-1}, but even when $p > n$, which we would expect in most computer vision applications, the difficulty of regularized ``inversion" of $\nabla_\theta \mathcal{L}_\theta$ relates to the non-linearity of the gradient operator as well as its conditioning.

Interestingly, fully-connected layers take a particular role in our problem: As we prove below, the input to a fully-connected layer can always be computed from the parameter gradients analytically independent of the layer's position in a neural network (provided that a technical condition, which prevents zero-gradients, is met). In particular, the analytic reconstruction is independent of the specific types of layers that precede or succeed the fully connected layer, and a single input to a fully-connected network can always be reconstructed analytically without solving an optimization problem.
The following statement is a generalization of Example 3 in \cite{phong_privacy-preserving_2017-1} to the setting of arbitrary neural networks with arbitrary loss functions:
\begin{proposition}
\label{prop:exactReconst}
Consider a neural network containing a biased fully-connected layer preceded solely by (possibly unbiased) fully-connected layers. Furthermore assume for any of those fully-connected layers the derivative of the loss $\mathcal L$ w.r.t.\ to the layer's output contains at least one non-zero entry. Then the input to the network can be reconstructed uniquely from the network's gradients.
\end{proposition}
\begin{proof}
In the following we give a sketch of the proof and refer to the supplementary material for a  more detailed derivation. Consider an unbiased full-connected layer mapping the input $x_l$ to the output, after e.g. a ReLU nonlinearity: $x_{l+1} = \max \{ A_l x_l, 0 \}$ for a matrix $A_l$ of compatible dimensionality. By assumption it holds $\frac{\diff \cl}{\diff \lbr x_{l+1}\rbr_i} \neq 0$ for some index $i$. Then by the chain rule $x_l$ can be computed as $\lbr \frac{\diff \cl}{\diff \lbr x_{l+1}\rbr_i}\rbr^{-1} \cdot \lbr \frac{\diff \cl}{\diff \lbr A_l \rbr_{i, \colon}}\rbr^T$. This allows the iterative computation of the layers' inputs as soon as the derivative of $\cl$ w.rt.\ a certain layer's output is known. We conclude by noting that adding a bias can be interpreted as a layer mapping $x_k$ to $x_{k+1} = x_k + b_k$ and that $\frac{\diff \cl}{\diff x_k} = \frac{\diff \cl}{\diff b_k}$.
\end{proof}

Another interesting aspect in view of the above considerations is that many popular network architectures use fully-connected layers (or cascades thereof) as their last prediction layers. Hence the input to those prediction modules being the output of the previous layers can be reconstructed. Those activations usually already contain some information about the input image thus exposing them to attackers.
Especially interesting in that regard is the possibility to reconstruct the ground truth label information from the gradients of the last fully-connected layer as discussed in \cite{zhao_idlg_2020}. Finally, Prop.~\ref{prop:exactReconst} allows to conclude that for any classification network that ends with a fully connected layer, reconstructing the input from a parameter gradient is strictly easier than inverting visual representations, as discussed in \cite{dosovitskiy_inverting_2016,dosovitskiy_generating_2016,mahendran_visualizing_2016}, from their last convolutional layer.

\section{A Numerical Reconstruction Method}\label{sec:method}
As image classification networks rarely start with fully connected layers, let us turn to the numerical reconstruction of inputs: Previous reconstruction algorithms relied on two components; the euclidean cost function of Eq. \eqref{eq:l2_penalty} and optimization via L-BFGS. We argue that these choices are not optimal for more realistic architectures and especially arbitrary parameter vectors. If we decompose a parameter gradient into its norm magnitude and its direction, we find that the magnitude only captures information about the state of training, measuring local optimality of the datapoint with respect to the current model. In contrast, the high-dimensional direction of the gradient can carry significant information, as the angle between two data points quantifies the change in prediction at one datapoint when taking a gradient step towards another \cite{charpiat_input_2019,koh_understanding_2017}. As such we propose to use a cost function based on angles, i.e. cosine similarity, $l(x,y) = \langle x, y\rangle / (||x|| ||y|||)$. In comparison to Eq. \eqref{eq:l2_penalty}, the objective is not to find images with a gradient that best fits the observed gradient, but to find images that lead to a similar change in model prediction as the (unobserved!) ground truth.
This is equivalent to minimizing the euclidean cost function, if one additionally constrains both gradient vectors to be normalized to a magnitude of 1.

We further constrain our search space to images within $[0,1]$ and add only total variation \cite{rudin_nonlinear_1992} as a simple image prior  to the overall problem, cf. \cite{wang_beyond_2018}:
\begin{equation}\label{eq:simrec}
    \arg \min_{x \in [0,1]^n} 1 - \frac{\langle \nabla_\theta \mathcal{L}_\theta(x, y), \nabla_\theta \mathcal{L}_\theta(x^*, y) \rangle}{||\nabla_\theta \mathcal{L}_\theta(x, y)|| ||\nabla_\theta \mathcal{L}_\theta(x^*, y)||} + \alpha \operatorname{TV}(x).
\end{equation}
Secondly, we note that our goal of finding some inputs $x$ in a given interval by minimizing a quantity that depends (indirectly, via their gradients) on the outputs of intermediate layers, is related to the task of finding adversarial perturbations for neural networks \cite{szegedy_intriguing_2013,madry_towards_2017,athalye_obfuscated_2018}. As such, we minimize eq. \eqref{eq:simrec} only based on the sign of its gradient, which we optimize with Adam \cite{kingma_adam:_2015} with step size decay. Note though that signed gradients only affect the first and second order momentum for Adam, with the actual update step still being unsigned based on accumulated momentum, so that an image can still be accurately recovered. 


Applying these techniques leads to the reconstruction observed in Fig. \ref{fig:teaser}. Further ablation of the proposed mechanism can be found in the appendix. We provide a \texttt{pytorch} implementation at \url{https://github.com/JonasGeiping/invertinggradients}.
\begin{remark}[Optimizing label information]
While we could also consider the label $y$ as unknown in Eq. \eqref{eq:simrec} and optimize jointly for $(x,y)$ as in \cite{zhu_deep_2019}, we follow \cite{zhao_idlg_2020} who find that label information can be reconstructed analytically for classification tasks. Thus, we consider label information to be known.
\end{remark}

\section{Single Image Reconstruction from a Single Gradient}\label{sec:arch}
Similar to previous works on breaking privacy in a federated learning setting, we first focus in the reconstruction of a single input image $x \in \R^n$ from the gradient $\nabla_\theta \mathcal{L}_\theta(x,y) \in \R^p$. This setting serves as a proof of concept as well as an upper bound on the reconstruction quality for the multi-image distributed learning settings we consider in Sec.~\ref{sec:realistic}. While previous works have already shown that a break of privacy is possible for single images, their experiments have been limited to rather shallow, smooth, and untrained networks. In the following, we compare our proposed approach to prior works, and conduct detailed experiments on the effect that architectural- as well as training-related choices have on the reconstruction. All
hyperparameter settings and more visual results for each experiment are provided in the supp. material.

\textbf{Comparison to previous approaches. } We first validate our approach by comparison to the Euclidean loss \eqref{eq:l2_penalty} optimized via L-BFGS considered in \cite{wang_beyond_2018,zhu_deep_2019,zhao_idlg_2020}. This approach can often fail due to a bad initialization, so we allow a generous setting of 16 restarts of the L-BFGS solver. For a quantitative comparison we measure the mean PSNR of the reconstruction of $32 \times 32$ CIFAR-10 images over the first 100 images from the validation set using the same shallow and smooth CNN as in \cite{zhu_deep_2019}, which we denote as "LeNet (Zhu)" as well as a ResNet architecture, both with trained and untrained parameters. Table \ref{tab:zhu} compares the reconstruction quality of euclidean loss \eqref{eq:l2_penalty} with L-BFGS optimization (as in \cite{wang_beyond_2018,zhu_deep_2019,zhao_idlg_2020}) with the proposed approach. The former works extremely well for the untrained, smooth, shallow architecture, but completely fails on the trained ResNet. We note that \cite{wang_beyond_2018} applied a GAN to enhance image quality from the LBFGS reconstruction, which, however, fails, when the representative is too distorted to be enhanced. Our approach provides recognizable images and works particularly well on the realistic setting of a trained ResNet as we can see in Figure~\ref{fig:baseline_comparison}. Interestingly, the reconstructions on the trained ResNet have a better visual quality than those of the untrained ResNet, despite their lower PSNR values according to table \ref{tab:zhu}. Let us study the effect of trained network parameters in an even more realistic setting, i.e., for reconstructing ImageNet images from a ResNet-152.

\begin{figure}[t]
    \begin{tabular}{lcccccc}
    \shortstack{\textbf{Eucl. Loss} +\\\textbf{L-BFGS}\\Untrained ResNet}
    &\includegraphics[width=0.1\textwidth]{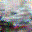} &\includegraphics[width=0.1\textwidth]{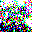} &\includegraphics[width=0.1\textwidth]{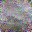} &\includegraphics[width=0.1\textwidth]{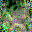} &\includegraphics[width=0.1\textwidth]{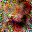} &\includegraphics[width=0.1\textwidth]{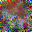} \\
    \shortstack{\textbf{Proposed}\\Untrained ResNet}  &\includegraphics[width=0.1\textwidth]{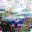} &\includegraphics[width=0.1\textwidth]{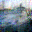} &\includegraphics[width=0.1\textwidth]{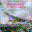} &\includegraphics[width=0.1\textwidth]{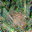} &\includegraphics[width=0.1\textwidth]{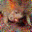} &\includegraphics[width=0.1\textwidth]{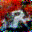} \\
    \hline\\
    
    \shortstack{\textbf{Euclidean Loss} +\\ \textbf{L-BFGS}\\Trained ResNet} &\includegraphics[width=0.1\textwidth]{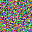} &\includegraphics[width=0.1\textwidth]{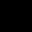} &\includegraphics[width=0.1\textwidth]{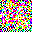} &\includegraphics[width=0.1\textwidth]{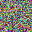} &\includegraphics[width=0.1\textwidth]{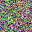} &\includegraphics[width=0.1\textwidth]{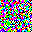} \\
    \shortstack{\textbf{Proposed}\\Trained ResNet}  &\includegraphics[width=0.1\textwidth]{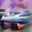} &\includegraphics[width=0.1\textwidth]{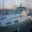} &\includegraphics[width=0.1\textwidth]{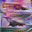} &\includegraphics[width=0.1\textwidth]{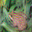} &\includegraphics[width=0.1\textwidth]{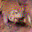} &\includegraphics[width=0.1\textwidth]{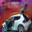}
    \end{tabular}
    \caption{Baseline comparison for the network architectures shown in \cite{wang_beyond_2018,zhu_deep_2019}.We show the first 6 images from the CIFAR-10 validation set.}
    \label{fig:baseline_comparison}
\end{figure}

\begin{table}
\caption{PSNR mean and standard deviation for 100 experiments on the first images of the CIFAR-10 validation data set over two different networks with trained an untrained parameters.}
\label{tab:zhu}
\centering
\begin{tabular}{lcccc}
\hline
Architecture & \multicolumn{2}{c}{LeNet (Zhu)} & \multicolumn{2}{c}{ResNet20-4} \\
Trained & False & True & False & True \\
\hline
Eucl. Loss + L-BFGS & $\boldsymbol{46.25 \pm 12.66}$ & $13.24 \pm 5.44 $& $10.29 \pm 5.38 $&$6.90 \pm 2.80$ \\
Proposed & $18.00 \pm 3.33$ & $\boldsymbol{18.08 \pm 4.27}$   &  $\boldsymbol{19.83 \pm 2.96}$ &  $\boldsymbol{13.95\pm 3.38}$  \\
\hline
\end{tabular}

\end{table}

\textbf{Trained vs. untrained networks. } If a network is trained and has sufficient capacity for the gradient of the loss function $\mathcal{L}_\theta$ to be zero for different inputs, it is obvious that they can never be distinguished from their gradient. In practical settings, however, owing to stochastic gradient descent, data augmentation and a finite number of training epochs, the gradient of images is rarely entirely zero. While we do observe that image gradients have a much smaller magnitude in a trained network than in an untrained one, our magnitude-oblivious approach of \eqref{eq:simrec} still recovers important visual information from the direction of the trained gradients. 

We observe two general effects on trained networks that we illustrate with our ImageNet reconstructions in Fig.~\ref{fig:imagenetstuff}: First, reconstructions seem to become implicitly biased to typical features of the same class in the training data, e.g., the more blueish feathers of the capercaillie in the 5th image, or the large eyes of the owl in our teaser \ref{fig:teaser}. Thus, although the overall privacy of most images is clearly breached, this effect at least obstructs the recovery of fine scale details or the image's background. 

Second, we find that the data augmentation used during the training of neural networks leads to trained networks that make the localization of objects more difficult: Notice how few of the objects in Fig. \ref{fig:imagenetstuff} retain their original position and how the snake and gecko duplicate. Thus, although image reconstruction with networks trained with data augmentation still succeeds, some location information is lost.

\begin{figure}
    \includegraphics[width=0.12\linewidth]{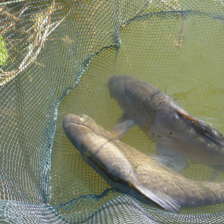}
    \includegraphics[width=0.12\linewidth]{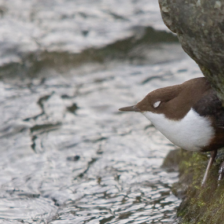}
    \includegraphics[width=0.12\linewidth]{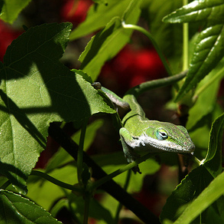}
    \includegraphics[width=0.12\linewidth]{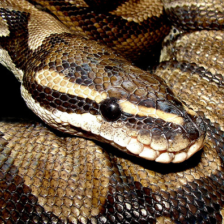}
    \includegraphics[width=0.12\linewidth]{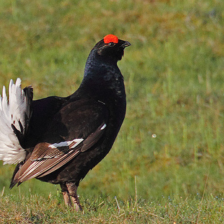}
    \includegraphics[width=0.12\linewidth]{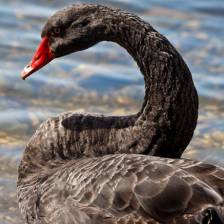}
    \includegraphics[width=0.12\linewidth]{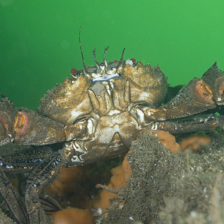}
    \includegraphics[width=0.12\linewidth]{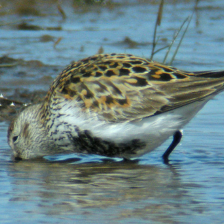}\\
    \includegraphics[width=0.12\linewidth]{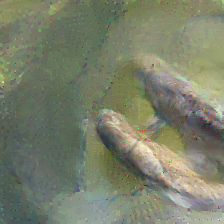}
    \includegraphics[width=0.12\linewidth]{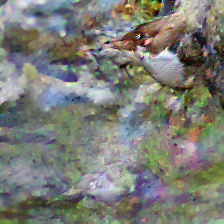}
    \includegraphics[width=0.12\linewidth]{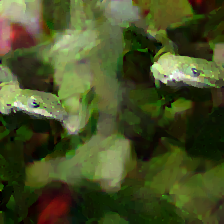}
    \includegraphics[width=0.12\linewidth]{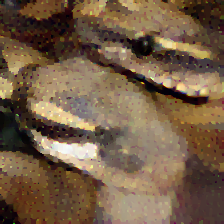}
    \includegraphics[width=0.12\linewidth]{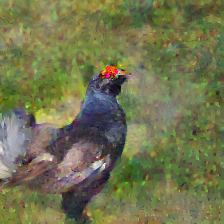}
    \includegraphics[width=0.12\linewidth]{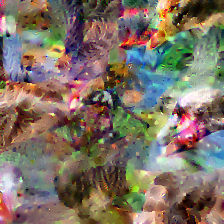}
    \includegraphics[width=0.12\linewidth]{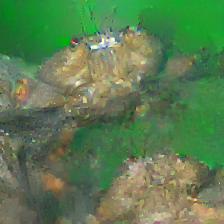}
    \includegraphics[width=0.12\linewidth]{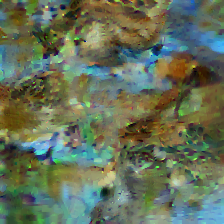}
    \caption{Single-Image Reconstruction from the parameter gradients of trained ResNet-152. Top row: Ground Truth. Bottom row: Reconstruction. We check every 1000th image of the ILSVRC2012 validation set. The amount of information leaked per image is highly dependent on image content - while some examples like the two tenches are highly compromised, the black swan leaks almost no usable information.}
    \label{fig:imagenetstuff}
\end{figure}

\begin{wrapfigure}[5]{r}{0.34\textwidth}
    \vspace{-15pt}
    \centering
    \includegraphics[width=0.14\textwidth]{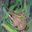} 
    \includegraphics[width=0.14\textwidth]{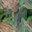} 
  \label{fig:circ_images}
\end{wrapfigure}\textbf{Translational invariant convolutions.} Let us study the ability to obscure the location of objects in more detail by testing how a conventional convolutional neural network, that uses convolutions with zero-padding, compares to a provably translationally invariant CNN, that uses convolutions with circular padding. As shown in the inset figure, while the conventional CNN allows for recovery of a rather high quality image (left), the translationally invariant network makes the localization of objects impossible (right) as the original object is separated. As such we identify the common zero-padding as a source of privacy risk.

\textbf{Network Depth and Width.} 
For classification accuracy, the depth and number of channels of each layer of a CNN are very important parameters, which is why we study their influence on our reconstruction. 
 Figure \ref{tab:depth_arch} shows that the reconstruction quality measurably increase with the number of channels. Yet, the larger network width is also accompanied with an increasing variance of experimental success.
 %
 However with multiple restarts of the experiment, better reconstructions can be produced for wider networks,
 resulting in PSNR values that increases from $19$ to almost $23$ for when increasing the number of channels from $16$ to $128$.%
 As such, greater network width increases the computational effort of the attacker, but does not provide greater security.

%

Looking at the reconstruction results we obtain from ResNets with different depths, the proposed attack degrades very little with an increased depth of the network. In particular - as illustrated in Fig.~\ref{fig:imagenetstuff}, even faithful ImageNet reconstructions through a ResNet-152 are possible. 
\newcommand\wid{1.8}
\newcommand\widless{1.7}
\begin{figure}
\centering
\begin{tabular}{p{\wid cm}p{\wid cm}p{\wid cm}p{\wid cm}p{\wid cm}p{\wid cm}}
\textbf{Original} & \multicolumn{3}{c}{\textbf{ResNet-18} with base width:} & \textbf{ResNet-34} & \textbf{ResNet-50}\\
 & 16 & 64 & 128  &  & \\ \hline 
\includegraphics[width=\widless cm]{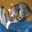}     & 
\includegraphics[width=\widless cm]{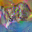}      & 
\includegraphics[width=\widless cm]{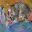}     & 
\includegraphics[width=\widless cm]{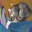}      & 
\includegraphics[width=\widless cm]{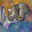}      & 
\includegraphics[width=\widless cm]{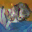}     \\ 
PSNR & 17.24 &17.37 & 25.25 &18.62& 21.36\\ 
Avg. PSNR     &19.02     &22.04     &22.94      & 21.59      & 20.98\\
Std.     &2.84     &5.89     &6.83     & 4.49      & 5.57 \\ 
&&&&\\[-0.5em]
\end{tabular}
\caption{Reconstructions of the original image (left) for multiple ResNet architectures. 
The PSNR value refers to the displayed image while the avg. PSNR is calculated over the first 10 CIFAR-10 images. The standard deviation is the average standard deviation of one experiment under a given architecture.
The ResNet-18 architecture is displayed for three different widths. 
}
\label{tab:depth_arch}
\end{figure}

\section{Distributed Learning with Federated Averaging and Multiple Images}\label{sec:realistic}
So far we have only considered recovery of a single image from its gradient and discussed limitations and possibilities in this setting. We now turn to strictly more difficult generalized setting of \emph{Federated Averaging} \cite{mcmahan_communication-efficient_2017,mcmahan_learning_2018,reddi_adaptive_2020} and \textit{multi-image} reconstruction, to show that the proposed improvements translate to this more practical case as well, discussing possibilities and limits in this application.

Instead of only calculating the gradient of a network's parameters based on local data, federated averaging performs multiple update steps on local data before sending the updated parameters back to the server. Following the notation of \cite{mcmahan_communication-efficient_2017}, we let the local data on the user's side consist of $n$ images. For a number $E$ of local epochs the user performs $\frac{n}{B}$ stochastic gradient update steps per epoch, where $B$ denotes the local mini-batch size, resulting in a total number of $E \frac{n}{B}$ local update steps. Each user $i$ then sends the locally updated parameters $\tilde \theta^{k+1}_i$ back to the server, which in turn updates the global parameters $\theta^{k+1}$ by averaging over all users.

\begin{figure}[bth]
\centering
    \begin{tabular}{cccc}
    \centering
   \ \  1 step, $\tau=$1e-4 \ \ & \ \ 100 steps, $\tau=$1e-4 \ \  & \ \ 5 steps, $\tau=$1e-1 \ \  &\ \  5 steps, $\tau=$1e-2 \ \  \\
    \includegraphics[width=0.15\textwidth]{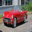}&
    \includegraphics[width=0.15\textwidth]{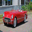}&
    \includegraphics[width=0.15\textwidth]{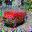}&
    \includegraphics[width=0.15\textwidth]{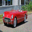}\\
    19.77dB & 19.39dB & 4.96dB & 23.74 dB
    \end{tabular}
    \caption{Illustrating the influence of the number of local update steps and the learning rate on the reconstruction: The left two images compare the influence of the number of gradient descent steps for a fixed learning rate of $\tau=$1e-4. The two images on the right result from varying the learning rate for a fixed number of 5 gradient descent steps. PSNR values are shown below the images.}
    \label{fig:fedavg}
\end{figure}

We empirically show that even the setting of federated averaging with $n\geq 1$ images is potentially amenable for attacks. To do so we try to reconstruct the local batch of $n$ images by the knowledge of the local update $\tilde \theta^{k+1}_i - \theta^k$. In the following we evaluate the quality of the reconstructed images for different choices of $n$, $E$ and $B$. We note that the setting studied in the previous sections corresponds to $n=1$, $E=1$, $B=1$. For all our experiments we use an untrained ConvNet.

\textbf{Multiple gradient descent steps, $B=n=1$, $E>1$:}\\
Fig.~\ref{fig:fedavg} shows the reconstruction of $n=1$ image for a varying number of local epochs $E$ and different choices of learning rate $\tau$. Even for a high number of 100 local gradient descent steps the reconstruction quality is unimpeded. The only failure case we were able to exemplify was induced by picking a high learning rate of 1e-1. This setup, however, corresponds to a step size that would lead to a divergent training update, and as such does not provide useful model updates.

\textbf{Multi-Image Recovery, $B=n>1$, $E=1$}:\\
So far we have considered the recovery of a single image only, and it seems reasonable to believe that averaging the gradients of multiple (local) images before sending an update to the server, restores the privacy of federated learning. While such a multi-image recovery has been considered in \cite{zhu_deep_2019} for $B\leq 8$, we demonstrate that the proposed approach is capable of restoring some information from a batch of 100 averaged gradients: While most recovered images are unrecognizable (as shown in the supplementary material), Fig. \ref{fig:multi_cifar} shows the 5 most recognizable images and illustrates that even averaging the gradient of 100 images does not entirely secure the private data. Most suprising is that the distortions arising from batching are not uniform. One could have expected all images to be equally distorted and near-irrecoverable, yet some images are highly distorted and others only to an extend at which the pictured object can still be recognized easily.
\begin{figure}[t]
    \centering
    \includegraphics[width=\textwidth]{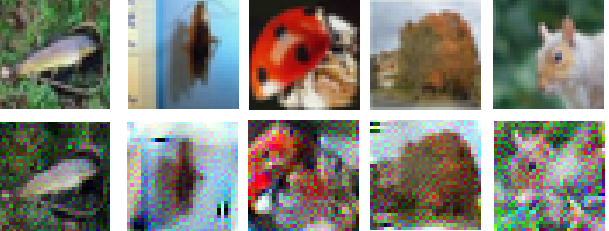}
    \caption{Information leakage for a batch of 100 images on CIFAR-100 for a ResNet32-10. Shown are the 5 \textit{most} recognizable images from the whole batch. Although most images are unrecognizable, privacy is broken even in a large-batch setting. We refer to the supplementary material for all images.}
    \label{fig:multi_cifar}
\end{figure}

\textbf{General case}\\
We also consider the general case of multiple local update steps using a subset of the whole local data in each mini batch gradient step. An overview of all conducted experiments is provided in Table \ref{tab:fedavg}. 
For each setting we perform 100 experiments on the CIFAR-10 validation set. For multiple images in a mini batch we only use images of different labels avoiding permutation ambiguities of reconstructed images of the same label. As to be expected, the single image reconstruction turns out to be most amenable to attacks in terms of PSNRs values. Despite a lower performance in terms of PSNR, we still observe privacy leakage for all multi-image reconstruction tasks, including those in which gradients in random mini-batches are taken. Comparing the full-batch, 8 images examples for 1 and 5 epochs, we see that our previous observation that multiple epochs do not make the reconstruction problem more difficult, extends to multiple images. For a qualitative assessment of reconstructed images of all experimental settings of Table~\ref{tab:fedavg}, we refer to the supplementary material. 


\begin{table}
\caption{PSNR statistics for various federated averaging settings, averaged over experiments on the first 100 images of the CIFAR-10 validation data set.}
\label{tab:fedavg}
\centering
\begin{tabular}{ccccccccc}
\hline
 \multicolumn{3}{c}{1 epoch} & \multicolumn{2}{c}{{5 epochs}}\\
\hline
 {4 images } &\multicolumn{2}{c}{{8 images}} & {1 image} & {8 images}\\
\hline
 batchsize 2 & batchsize 2 & batchsize 8 & batchsize 1 & batchsize 8 \\
\hline
\hline
 $\boldsymbol{16.92 \pm  2.10}$  &  $\boldsymbol{14.66 \pm 1.12  }$ &  $\boldsymbol{16.49\pm 1.02  }$ &  $\boldsymbol{25.05  \pm  3.28 }$  &  $\boldsymbol{ 16.58 \pm 0.96  }$ \\
\\
\end{tabular}
\end{table}

\section{Conclusions}
Federated learning is a modern paradigm shift in distributed computing, yet its benefits to privacy are not as well understood yet. We shed light into possible avenues of attack, analyzed the ability to reconstruct the input to any fully connected layer analytically, propose a general optimization-based attack, and discuss its effectiveness for different types of architectures 
and network parameters. Our experimental results are obtained with modern computer vision architectures for image classification. They clearly indicate that provable differential privacy remains the only way to guarantee security, possibly even for larger batches of data points.

\section*{Broader Impact - Federated Learning does not guarantee privacy}
Recent works on privacy attacks in federated learning setups (\cite{phong_privacy-preserving_2017,phong_privacy-preserving_2017-1,wang_beyond_2018,zhu_deep_2019, zhao_idlg_2020}) have hinted at the fact that previous hopes that
``Privacy is enhanced by the ephemeral and focused nature of the [Federated Learning] updates'' \cite{bonawitz_towards_2019} are not true in general. In this work, we demonstrated that improved optimization strategies such as a cosine similarity loss and a signed Adam optimizer allow for image recovery in a federated learning setup in industrially realistic settings for computer vision: Opposed to the idealized architectures of previous works we demonstrate that image recovery is possible in deep non-smooth architectures over multiple federated averaging steps of the optimizer and even in batches of 100 images.

We note that image classification is possibly especially vulnerable to these types of attacks, given the inherent structure of image data, the size of image classification networks, and the comparatively small number of images a single user might own, relative to other personal information. On the other hand, this attack is likely only a first step towards stronger attacks.
Therefore, this work points out that the question how to protect the privacy of our data while collaboratively training highly accurate machine learning approaches remains largely unsolved: While differential privacy offers provable guarantees, it also reduces the accuracy of the resulting models significantly \cite{jayaraman_evaluating_2019-1}. As such differential privacy and secure aggregation can be costly to implement so that there is some economic incentive for data companies to use only basic federated learning. For a more general discussion, see \cite{veale_algorithms_2018}. Thus, there is strong interest in further research on privacy preserving learning techniques that render the attacks proposed in this paper ineffective. This might happen via defensive mechanisms or via computable guarantees that allow practitioners to verify whether their application is vulnerable to such an attack. 


%
%
%
%



{\small
\bibliographystyle{plain}
\bibliography{zotero_library,additional}
}

\appendix 

\section{Variations of the threat model}\label{sec:threatmodel}
In this work we consider a \textit{honest-but-curious} threat model as discussed in the introduction. Straying from this scenario could be done primarily in two ways: First by changing the architecture, and second by keeping the architecture non-malicious, but changing the global parameters sent to the user.

\subsection{Dishonest Architectures}
So far we assumed that the server operates under an \textit{honest-but-curious} model, and as such would not modify the model maliciously to make reconstruction easier. If we instead allow for this, then reconstruction becomes nearly trivial:
Several mechanisms could be used: Following Prop. 1, the server could, for example, place a fully-connected layer in the first layer, or even directly connect the input to the end of the network by concatenation. Slightly less obvious, the model could be modified to contain reversible blocks \cite{chang_reversible_2017,jacobsen_i-revnet_2018}. These blocks allow the recovery of input from their outputs. From Prop. 1 we know that we can reconstruct the input to the classification layer, so this allows for immediate access to the input image. If the server maliciously introduces separate weights or sub-models for each batch example, then this also allows for a recovery of an arbitrarily large batch of data.
Operating in a setting, where such behavior is possible would require the user (or a provider trusted by the user) to vet any incoming model either manually or programmatically. 

\subsection{Dishonest Parameter Vectors}
However, even with a fixed \textit{honest} architecture, a malicious choice of global parameters can significantly influence reconstruction quality. For example, considering the network architecture in \cite{zhu_deep_2019} which does not contain strides and flattens convolutional features, the dishonest server could set all convolution layers to represent the identity \cite{goldblum_truth_2019}, moving the input through the network unchanged up to the classification layer, from which the input can be analytically computed as in Prop. 1. Likewise for an architecture that contains strides to a recognizable lower resolution \cite{wang_beyond_2018}, the input can be recovered immediately albeit in a smaller resolution when the right parameter vector is sent to the user.

Such a specific choice of parameters is however likely detectable. A subtler approach, as least possible in theory, would be to optimize the network parameters themselves that are sent to the user so that reconstruction quality from these parameters is maximized. While such an attack is likely to be difficult to detect on the user-side, it would also be very computationally intensive.

\textbf{Label flipping.}
There is even a cheaper alternative. According to Sec. 5, very small gradient vectors may contain less information. A simple way for a dishonest server to boost these gradients is to permute two rows in the weight matrix and bias of the classification layer, effectively flipping the semantic meaning of a label. This attack is difficult to detect for the user (as long as the gradient magnitude stays within usual bounds), but effectively tricks him into differentiating his network w.r.t to the wrong label. Fig. \ref{fig:label_flip} shows that this mechanism can allow for a reliable reconstruction with boosted PSNR scores, as the effect of the trained model is negated.
\begin{figure}
    \centering
        \begin{tabular}{c | c |c |c |c |c |c |c }
    \includegraphics[width=0.1\textwidth]{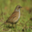} &
    \multicolumn{1}{c}{\includegraphics[width=0.1\textwidth]{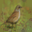}} &
    \includegraphics[width=0.1\textwidth]{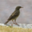} &
    \multicolumn{1}{c}{\includegraphics[width=0.1\textwidth]{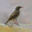}} &
    \includegraphics[width=0.1\textwidth]{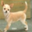} &
    \multicolumn{1}{c}{\includegraphics[width=0.1\textwidth]{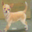}} &
    \includegraphics[width=0.1\textwidth]{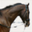} &
    \multicolumn{1}{c}{\includegraphics[width=0.1\textwidth]{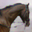}} \\
    5.9e-1 &  29.37dB & 4.6e+2 & 26.62dB & 1.8e+2 &  27.37dB & 1.5e+2 & 18.27dB \\
    \end{tabular}
    \caption{Label flipping. Images can be easily reconstructed when two rows in the parameters of the final classification layer are permuted. Below each input image is given the gradient magnitude, below each output image its PSNR. Compare these results to the additional examples in Fig. \ref{fig:training_images_vs_val_images}}
    \label{fig:label_flip}
\end{figure}

\section{Experimental Details}\label{sec:topk}


\begin{figure}[h]
    \centering
    \includegraphics[width=0.2\textwidth]{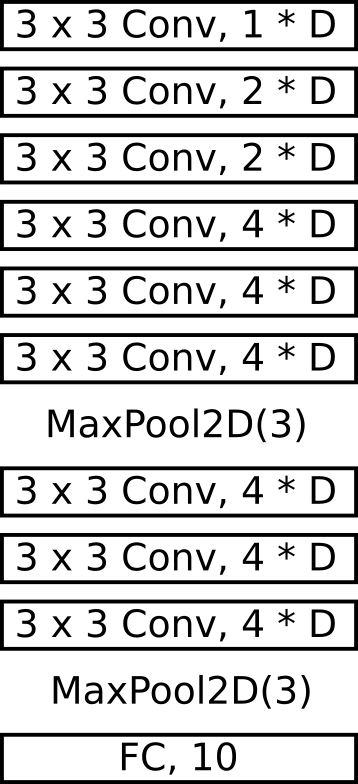}
    \caption{Network architecture \textit{ConvNet}, consisting of 8 convolution layers, specified with corresponding number of output channels. Each convolution layer is followed by a batch normalization layer and a ReLU layer. $D$ scales the number of output channels and is set to $D = 64$ by default.}
    \label{fig:convnet_arch}
\end{figure}

\subsection{Federated Averaging}
The extension of Eq. (4) to the case of federated averaging (in which multiple local update steps are taken and sent back to the server) is straightforward. Notice first, that given old parameters $\theta^k$, local updates $\theta^{k+l}$, learning rate $\tau$, and knowledge about the number of update steps\footnote{We assume that the number of local updates is known to the server, yet this could also be found by brute-force, given that $l$ is a small integer.}, the update can be rewritten as the average of updated gradients.
\begin{equation}
    \theta^{k+l} = \theta^k - \tau \sum_{m=1}^l \nabla_{\theta^{k+m}} \mathcal{L}_{\theta^{k+m}}(x, y)
\end{equation}
Subtracting $\theta^k$ from $\theta^{k+l}$, we simply apply the proposed approach to the resulting average of updates:
\begin{equation}\label{eq:simrec_avg}
    \arg \min_{x \in [0,1]^n} 1 - \frac{ \langle \sum_{m=1}^l \nabla_{\theta^{k+m}} \mathcal{L}_{\theta^{k+m}}(x, y), \sum_{m=1}^l \nabla_{\theta^{k+m}} \mathcal{L}_{\theta^{k+m}}(x^*, y) \rangle}{||\sum_{m=1}^l \nabla_{\theta^{k+m}} \mathcal{L}_{\theta^{k+m}}(x, y)|| ||\sum_{m=1}^l \nabla_{\theta^{k+m}} \mathcal{L}_{\theta^{k+m}}(x^*, y)||} + \alpha \operatorname{TV}(x).
\end{equation}
Using automatic differentiation, we backpropagate the gradient w.r.t to $x$ from the average of update steps.

\subsection{ConvNet}
We use a ConvNet architecture as a baseline for our experiments as it is relatively fast to optimize, reaches above 90\% accuracy on CIFAR-10 and includes two max-pooling layers. It is a rough analogue to AlexNet \cite{krizhevsky_imagenet_2012}. The architecture is described in Fig. \ref{fig:convnet_arch}.
\begin{table}
\centering
\caption{\label{table:ablation} Ablation Study for the proposed approach for a trained ResNet-18 architecture, trained on CIFAR-10. Reconstruction PSNR scores are averaged over the first 10 images of the CIFAR-10 validation set (Standard Error in parentheses).}
\begin{tabular}{|c | c|}
\hline
Basic Setup         & $20.12$ dB  ($\pm 1.02$)\\ \hline
L2 Loss instead of cosine similarity & $15.13$ dB ($\pm 0.70$) \\\hline
Without total variation                & $19.96$ dB  ($\pm 0.75$)  \\ \hline
With L-BFGS instead of Adam  & $5.13$ dB  ($\pm 0.50$) \\ \hline
\end{tabular} 
\end{table}
\subsection{Ablation Study}
We provide an ablation for proposed choices in Table \ref{table:ablation}. We note that two things are central, the Adam optimizer and the similarity loss. Total variation is a small benefit, and using signed gradients is a minor benefit. 

\section{Hyperparameter Settings}\label{sec:hyperparam}
In our experiments we reconstruct the network's input using Adam based on signed gradients as optimization algorithm and cosine similarity as cost function as described in Sec. 4. 
It is important to note that the optimal hyperparameters for the attack depend on the specific attack scenario - that the attack fails with default parameters is no guarantee for security. We always initialize our reconstructions from a Gaussian distribution with mean 0 and variance 1 (Note that the input data is normalized as usual for all considered datasets) and set the step size of the optimization algorithm within $[0.01, 1]$. We use a smaller step sizes of $0.1$, for the wider and deeper networks in Sec. 5.2 and a larger step sizes of $1$ for the federated averaging experiments in Sec 6, with $0.1$ being the default choice. The optimization runs for up to $24 000$ iterations. The step size decay is always fixed, occuring after $\frac{3}{8}$, $\frac{5}{8}$ and $\frac{7}{8}$ of iterations and reducing the learning rate by a factor of $0.1$ each time. The number of iterations is a generally conservative estimate, privacy can often be broken much earlier.

We tweak the total variation parameter depending on the specific attack scenario, however note that its effect on avg. PSNR is mostly minor as seen in table \ref{table:ablation}. When not otherwise noted we default to a value of $0.01$.  
\begin{remark}[Restarts]
Generally, multiple restarts of the attack from different random initializations can improve the attack success moderately. However they also increase the computational requirements significantly. To allow for quantitative experimental evaluations of multiple images, we do not consider restarts in this work (aside from Sec. 5 where we apply them to improve results of the competing LBFGS solver) - but stress that an attacker with enough ressources could further improve his attack by running it with multiple restarts. 
\end{remark} 



\subsection{Settings for the experiments in Sec.\;5}



\paragraph{Comparison to previous approaches}

For comparison with baselines in section 5, we re-implement the network from \cite{zhu_deep_2019}, which we dub LeNet (Zhu) in the following, and additionally run all experiments for the ResNet20-4 architecture. We base both the network and the approach on code from the authors of \cite{zhu_deep_2019}, \footnote{\url{https://github.com/mit-han-lab/dlg}}. For the LBFGS-L2 optimization we use a learning rate of $1e-4$ and $300$ iterations. For the ResNet experiments we use the generous amount of $8$ restarts and for the faster to optimize LeNet (Zhu) architecture we use the even higher number of $16$ restarts. All experiment conducted with the proposed approach only use one restart, $4800$ iterations, a learning rate of $0.1$ and TV regularization parameters as detailed in Table \ref{tab:baseline}. Note that in the described settings the proposed method took significantly less time to optimize than the LBFGS optimization.
\paragraph{Spatial Information}
The experiments on spatial information are performed on the ConvNet architecture with $D =64$ channels.

\begin{table}[]
    \centering
    \begin{tabular}{ccccc}
    \hline
     Architecture & \multicolumn{2}{c}{LeNet (Zhu) } & \multicolumn{2}{c}{ResNet20-4} \\
     \hline
     Trained & False & True & False & True\\
     \hline\hline
     TV & $10^{-2}$ & $10^{-3}$ & $0$ & $10^{-2}$\\
     \hline
    \end{tabular}
    \caption{TV regularization values used for the proposed approach in the baseline experiments of Section 5.}
    \label{tab:baseline}
\end{table}

\subsection{Setting for experiments in Sec. 6}
For the five cases consider in Table 2 we consider an untrained ConvNet, a learning rate of $1$, $4800$ iterations, one restart and the TV regularization parameters as given in table \ref{table:fedavg}. Each of the $100$ experiments uses different images, i.e.\ each experiments uses the images of the CIFAR-10 validation set following the ones used in the previous experiment. As multiple images of the same label in one mini-batch cause an ambiguity in the ordering of images w.r.t.\ that label, we do not consider that case. If an image with an already encountered label is about to be added to the respective mini-batch we skip that image and use the next image of the validation set with a different label. 
\begin{table}
\centering
\begin{tabular}{p{3.6cm} ccccc}
\hline
Number of epochs $E$               &1&1&1&5&5    \\ \hline
Number of local images $n$               &4&8&8&1&8     \\ \hline
Mini-batch size $B$               &2&2&8&1&8    \\ \hline \hline
TV              &$10^{-6}$&$10^{-6}$&$10^{-4}$&$10^{-4}$&$10^{-4}$
\end{tabular}
\caption{\label{table:fedavg} Total variation weights for the reconstruction of network input in the experiments in Sec.\;4.2}
\end{table}
\section{Proofs for section 3.1}\label{sec:proofs}
In the following we give a more detailed proof of Prop 3.1, which is follows directly from the two propositions below:
\setcounter{proposition}{0}
\begin{proposition}\label{prop:bfc_repeat}
Let a neural network contain a biased fully-connected layer at some point, i.e.\ for the layer's input $x_l\in\R^{n_l}$ its output $x_{l+1}\in\R^{n_{l+1}}$ is calculated as $x_{l+1} = \max \{ y_l, 0 \}$ for
\begin{align}
    y_l = A_l x_l + b_l,
\end{align}
for $A_l \in \R^{n_{l+1}\times n_l}$ and $b_l \in \R^n_{l+1}$. Then the input $x_l$ can be reconstructed from $\frac{\diff \cl}{\diff A_l}$ and $\frac{\diff \cl}{\diff b_l}$, if there exists an index $i$ s.t.\ $\frac{\diff \cl}{\diff \lbr b_l \rbr_i} \neq 0$.
\end{proposition}
\begin{proof}
It holds that $\frac{\diff \cl}{\diff \lbr b_l \rbr_i} = \frac{\diff \cl}{\diff \lbr y_l \rbr_i}$ and $\frac{\diff y_i}{\diff \lbr A_l \rbr_{i,\colon}} = x^T$. Therefore
\begin{align}
    \frac{\diff \cl}{\diff \lbr A_l \rbr_{i,\colon}} &= \frac{\diff \cl}{\diff \lbr y_l \rbr_i} \cdot \frac{\diff \lbr y_l \rbr_i}{\diff \lbr A_l \rbr_{i,\colon}}\\
    &= \frac{\diff \cl}{\diff \lbr b_l \rbr_i} \cdot x_l^T
\end{align}
for $\lbr A_l \rbr_{i,\colon}$ denoting the $i$\textsuperscript{th} row of $A_l$. Hence $x_l$ can can be uniquely determined as soon as $\frac{\diff \cl}{\diff \lbr b_l \rbr_i} \neq 0$.
\end{proof}

\begin{proposition}\label{prop:fc_repeat}
Consider a fully-connected layer (not necessarily including a bias) followed by a ReLU activation function, i.e.\ for an input $x_l \in \R^{n_l}$ the output $x_{l+1} \in \R^{n_{l+1}}$ is calculated as $x_{l+1} = \max \{ y_l, 0 \}$ for
\begin{align}
    y_l = A_l x_l,
\end{align}
where the maximum is computed element-wise. Now assume we have the additional knowledge of the derivative w.r.t.\ to the output $\frac{\diff \cl}{\diff x_{l+1}}$. Furthermore assume there exists an index $i$ s.t.\ $\frac{\diff \cl}{\diff \lbr x_{l+1} \rbr_i} \neq 0$. Then the input $v$ can be derived from the knowledge of $\frac{\diff \cl}{\diff A_l}$.
\end{proposition}
\begin{proof}
As $\frac{\diff \cl}{\diff \lbr x_{l+1} \rbr_i} \neq 0$ it holds that $\frac{\diff \cl}{\diff \lbr y_l \rbr_i} = \frac{\diff \cl}{\diff \lbr x_{l+1} \rbr_i}$ and it follows that
\begin{align}
    \frac{\diff \cl}{\diff \lbr A_l \rbr_{i,\colon}} &= \frac{\diff \cl}{\diff \lbr y_l \rbr_i} \cdot \frac{\diff \lbr y_l \rbr_i}{\diff \lbr A_l \rbr_{i,\colon}}\\
    &= \frac{\diff \cl}{\diff \lbr x_{l+1}\rbr_i} \cdot x_l^T.
\end{align}
\end{proof}

\section{Additional Examples}\label{sec:examples}

\subsection{Additional CIFAR-10 examples}
Figure \ref{fig:training_images_vs_val_images} shows additional "extreme" examples for CIFAR-10, reconstructing the image with lowest and the image with largest gradient magnitude for the training and validation set of CIFAR-10 for trained and untrained ConvNet and ResNet20-4 models.
\begin{figure}[t]
    \centering
    \begin{tabular}{c|c|c|c|c|c|c|c}
    \multicolumn{8}{c}{\textbf{Trained ConvNet}} \\
		\multicolumn{4}{c}{Images from the training set } & \multicolumn{4}{c}{Images from the validation set} \\
     \includegraphics[width=0.1\textwidth]{eccv2020kit/images/parameters/ConvNet64_input_low_train.png}  &  
     \includegraphics[width=0.1\textwidth]{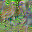} &
     \includegraphics[width=0.1\textwidth]{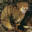} &
     \includegraphics[width=0.1\textwidth]{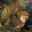}  &
     \includegraphics[width=0.1\textwidth]{eccv2020kit/images/parameters/ConvNet64_input_low_val.png} &
     \includegraphics[width=0.1\textwidth]{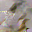}&
     \includegraphics[width=0.1\textwidth]{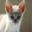} &
     \includegraphics[width=0.1\textwidth]{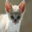}
     \\
     4.5e-21 & 18.04dB & 2.5e+02 & 14.85dB & 9.8e-17 & 14.60dB & 5.5e+02 & 30.26dB \\
     \multicolumn{8}{c}{}\\
     \multicolumn{8}{c}{\textbf{Trained ResNet20-4}}\\
		\multicolumn{4}{c}{Images from the training set } & \multicolumn{4}{c}{Images from the validation set} \\
     \includegraphics[width=0.1\textwidth]{eccv2020kit/images/parameters/ResNet20-4_input_low_train.png}  &  
     \includegraphics[width=0.1\textwidth]{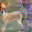} &
     \includegraphics[width=0.1\textwidth]{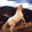} &
     \includegraphics[width=0.1\textwidth]{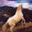}  &
     \includegraphics[width=0.1\textwidth]{eccv2020kit/images/parameters/ResNet20-4_input_low_val.png} &
     \includegraphics[width=0.1\textwidth]{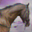}&
     \includegraphics[width=0.1\textwidth]{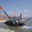} &
     \includegraphics[width=0.1\textwidth]{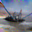}
     \\
     5.3-06 & 15.21dB & 1.0e+2 & 19.75dB & 1.2e-5 & 13.84dB & 4.6e+2 & 15.53dB \\
 \multicolumn{8}{c}{}\\
    \multicolumn{4}{c}{\textbf{Untrained ConvNet}} & \multicolumn{4}{c}{\textbf{Untrained ResNet20-4}} \\
     \includegraphics[width=0.1\textwidth]{eccv2020kit/images/parameters/ConvNet64_input_low_train.png} &
     \includegraphics[width=0.1\textwidth]{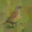} &
     \includegraphics[width=0.1\textwidth]{eccv2020kit/images/parameters/ConvNet64_input_low_val.png} &
     \includegraphics[width=0.1\textwidth]{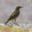}  &  
    \includegraphics[width=0.1\textwidth]{eccv2020kit/images/parameters/ResNet20-4_input_low_train.png} &
    \includegraphics[width=0.1\textwidth]{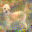} &
    \includegraphics[width=0.1\textwidth]{eccv2020kit/images/parameters/ResNet20-4_input_low_val.png} &
    \includegraphics[width=0.1\textwidth]{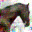} \\
    6.1e-1 &  31.36dB & 6.7e-1 & 31.16dB & 3.8e+1 & 21.90dB & 4.5e+1 & 20.23dB \\
    \end{tabular}
    \caption{Reconstruction of images for the \textit{trained} ConvNet model (Top) and ResNet20-4 (middle). We show reconstructions of the \textbf{worst-case} image and \textbf{best case} image from CIFAR-10, based on gradient magnitude for both the training and the validation set. Below each input image is given the gradient magnitude, below each output image its PSNR. The bottom row shows reconstructions for the worst-case examples for untrained models.}
    \label{fig:training_images_vs_val_images}
\end{figure}

\subsection{Visualization of experiments in Sec. 5}
\paragraph{Network Width}
The reconstructions for the first six CIFAR images for different width ResNet-18 architectures are given in Fig.\;\ref{tab:width_all_results}.

\newcommand\widw{1.1}
\newcommand\widlessw{1.4}
\begin{figure}[h!]
\centering
\begin{tabular}{l P{\widw cm}P{\widw cm}P{\widw cm}P{\widw cm}P{\widw cm}P{\widw cm}P{\widw cm}}
16 Channels  &  
\includegraphics[width=\widlessw cm]{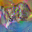}&  
\includegraphics[width=\widlessw cm]{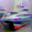}&  
\includegraphics[width=\widlessw cm]{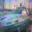}&  
\includegraphics[width=\widlessw cm]{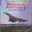}&  
\includegraphics[width=\widlessw cm]{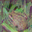}&  
\includegraphics[width=\widlessw cm]{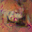}&  
\includegraphics[width=\widlessw cm]{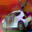}\\
64 Channels  &  
\includegraphics[width=\widlessw cm]{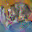}&
\includegraphics[width=\widlessw cm]{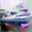}&
\includegraphics[width=\widlessw cm]{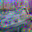}&
\includegraphics[width=\widlessw cm]{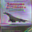}&
\includegraphics[width=\widlessw cm]{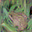}&
\includegraphics[width=\widlessw cm]{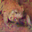}&
\includegraphics[width=\widlessw cm]{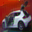}\\
128 Channels & 
\includegraphics[width=\widlessw cm]{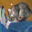}&
\includegraphics[width=\widlessw cm]{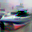}&
\includegraphics[width=\widlessw cm]{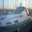}&
\includegraphics[width=\widlessw cm]{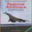}&
\includegraphics[width=\widlessw cm]{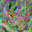}&
\includegraphics[width=\widlessw cm]{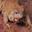}&
\includegraphics[width=\widlessw cm]{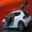}\\
\end{tabular}
\caption{Reconstructions using ResNet-18 architectures with different widths.}
\label{tab:width_all_results}
\end{figure}

\paragraph{Network Depth}
The experiments concerning the network depth are performed for different deep ResNet architectures. Multiple reconstruction results for different deep networks are shown in Fig.\;\ref{tab:depth_all_results}. 

\begin{figure}[h!]
\centering
\begin{tabular}{l P{\widw cm}P{\widw cm}P{\widw cm}P{\widw cm}P{\widw cm}P{\widw cm}P{\widw cm}}
ResNet-18  &  
\includegraphics[width=\widlessw cm]{neurips2020/image_width_depth_new/images_width/2469742211_trainedResNet18-FM_sim-def_iter-0.png}&  
\includegraphics[width=\widlessw cm]{neurips2020/image_width_depth_new/images_width/2469742211_trainedResNet18-FM_sim-def_iter-1.png}&  
\includegraphics[width=\widlessw cm]{neurips2020/image_width_depth_new/images_width/2469742211_trainedResNet18-FM_sim-def_iter-2.png}&  
\includegraphics[width=\widlessw cm]{neurips2020/image_width_depth_new/images_width/2469742211_trainedResNet18-FM_sim-def_iter-3.png}&  
\includegraphics[width=\widlessw cm]{neurips2020/image_width_depth_new/images_width/2469742211_trainedResNet18-FM_sim-def_iter-4.png}&  
\includegraphics[width=\widlessw cm]{neurips2020/image_width_depth_new/images_width/2469742211_trainedResNet18-FM_sim-def_iter-5.png}&  
\includegraphics[width=\widlessw cm]{neurips2020/image_width_depth_new/images_width/2469742211_trainedResNet18-FM_sim-def_iter-6.png}\\
ResNet-34 &  
\includegraphics[width=\widlessw cm]{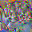}&
\includegraphics[width=\widlessw cm]{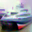}&
\includegraphics[width=\widlessw cm]{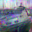}&
\includegraphics[width=\widlessw cm]{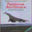}&
\includegraphics[width=\widlessw cm]{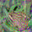}&
\includegraphics[width=\widlessw cm]{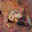}&
\includegraphics[width=\widlessw cm]{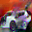}\\
ResNet-50 & 
\includegraphics[width=\widlessw cm]{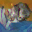}&
\includegraphics[width=\widlessw cm]{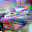}&
\includegraphics[width=\widlessw cm]{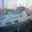}&
\includegraphics[width=\widlessw cm]{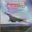}&
\includegraphics[width=\widlessw cm]{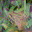}&
\includegraphics[width=\widlessw cm]{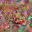}&
\includegraphics[width=\widlessw cm]{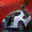}\\
\end{tabular}
\caption{Reconstructions using different deep ResNet architectures.}
\label{tab:depth_all_results}
\end{figure}



\subsection{More ImageNet examples for Sec. 5}
Fig. \ref{fig:moreImageNet} shows further instructive examples of reconstructions for ImageNet validation images for a trained ResNet-18 (the same setup as Fig. 3 in the main paper). We show a very good reconstruction (German shepherd), a good, but translated reconstruction (giant panda) and two failure cases (ambulance and flower). For the ambulance, for example, the actual writing on the ambulance car is still hidden. For the flower, the exact number of petals is hidden. Also, note how the reconstruction of the giant panda is much clearer than that of the tree stump co-occurring in the image, which we consider an indicator of the self-regularizing effect described in Sec. 5.

Figures \ref{fig:imagenetstuff3} and \ref{fig:imagenetstuff2} show more examples. We note that the examples in these figures and in Figure 3 are not handpicked, but chosen neutrally according to their ID in the ILSVRC2012, ImageNet, validation set. The ID for each image is obtained by sorting the synset that make up the dataset in increasing order according to their synset ID and sorting the images within each synset according to their synset ID in increasing order. This is the default order in \texttt{torchvision}.

\begin{figure}
    \centering
    \includegraphics[width=0.22\textwidth]{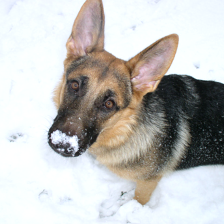}
    \includegraphics[width=0.22\textwidth]{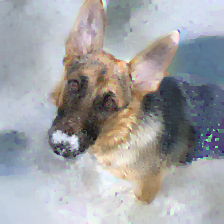}
    \includegraphics[width=0.22\textwidth]{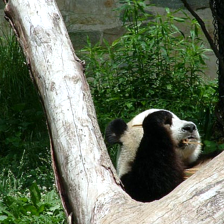}
    \includegraphics[width=0.22\textwidth]{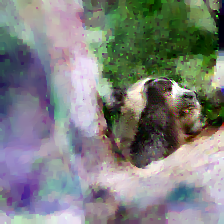}\\
    \includegraphics[width=0.22\textwidth]{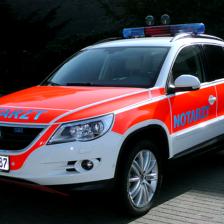}
    \includegraphics[width=0.22\textwidth]{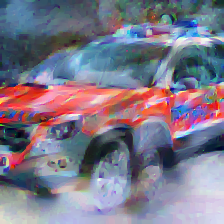}
    \includegraphics[width=0.22\textwidth]{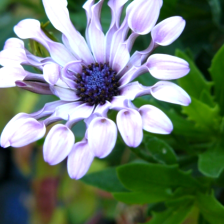}
    \includegraphics[width=0.22\textwidth]{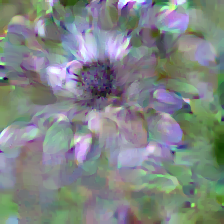}
    \caption{Additional qualitative ImageNet examples, failure cases and positive cases for a trained ResNet-18. Images taken from the ILSVRC2012 validation set.}
    \label{fig:moreImageNet}
\end{figure}

\begin{figure}
\centering
    \includegraphics[width=0.19\textwidth]{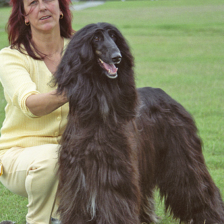}
    \includegraphics[width=0.19\textwidth]{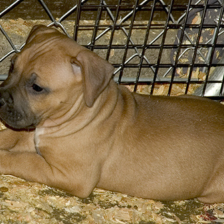}
    \includegraphics[width=0.19\textwidth]{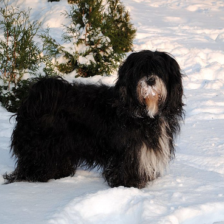}
    \includegraphics[width=0.19\textwidth]{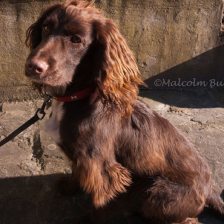}
    \includegraphics[width=0.19\textwidth]{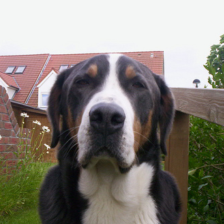}\\
    \includegraphics[width=0.19\textwidth]{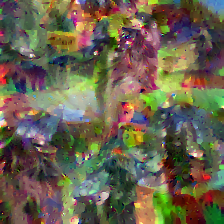}
    \includegraphics[width=0.19\textwidth]{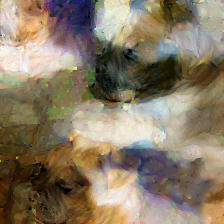}
    \includegraphics[width=0.19\textwidth]{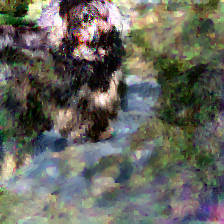}
    \includegraphics[width=0.19\textwidth]{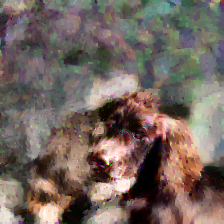}
    \includegraphics[width=0.19\textwidth]{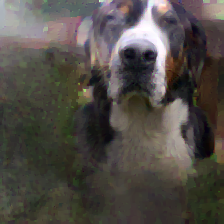}
    \caption{Additional single-image reconstruction from the parameter gradients of trained ResNet-152. Top row: Ground Truth. Bottom row: Reconstruction. The paper showed images 0000, 1000, 2000, 3000, 4000, 5000, 6000, 7000 from the ILSVRC2012 validation set. These are images 8000-12000.}
    \label{fig:imagenetstuff3}
\end{figure}

\begin{figure}
\centering
    \includegraphics[width=0.19\textwidth]{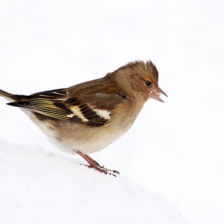}
    \includegraphics[width=0.19\textwidth]{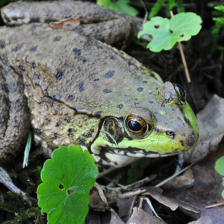}
    \includegraphics[width=0.19\textwidth]{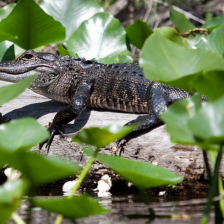}
    \includegraphics[width=0.19\textwidth]{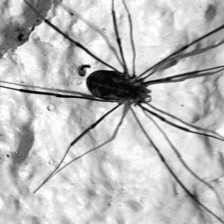}
    \includegraphics[width=0.19\textwidth]{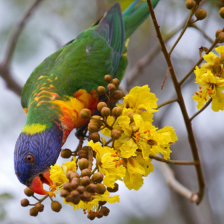}\\
    \includegraphics[width=0.19\textwidth]{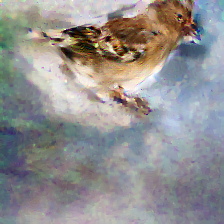}
    \includegraphics[width=0.19\textwidth]{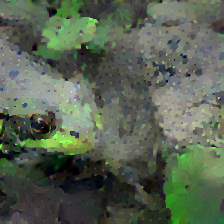}
    \includegraphics[width=0.19\textwidth]{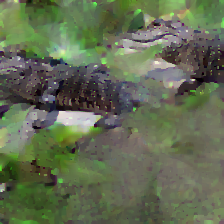}
    \includegraphics[width=0.19\textwidth]{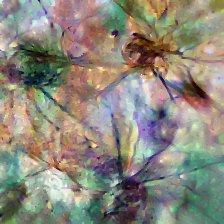}
    \includegraphics[width=0.19\textwidth]{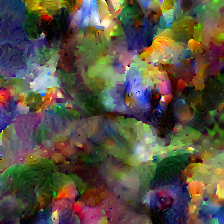}
    \caption{Additional single-image reconstruction from the parameter gradients of trained ResNet-152. Top row: Ground Truth. Bottom row: Reconstruction. These are images 500, 1500, 2500, 3500, 4500.}
    \label{fig:imagenetstuff2}
\end{figure}

\subsection{Multi-Image Recovery of Sec. 6}
For multi-image recovery, we show the full set of 100 images in Fig. \ref{fig:full_cifar100_2}, we recommend to zoom in to a digital version of the figure. The success rate for separate images is semi-random, depending on the initialization. 

\subsection{General case of Sec. 6}
We show the results for the first ten experiments in Figures \ref{fig:fedavg4}, \ref{fig:fedavg1}, \ref{fig:fedavg2}, \ref{fig:fedavg3}, \ref{fig:fedavg5}. In Figure \ref{fig:fedavg4} we even show all $100$ experiments as there only one image is used per experiment.

\begin{figure}
    \centering
\begin{tabular}{l P{\widw cm}P{\widw cm}P{\widw cm}P{\widw cm}P{\widw cm}P{\widw cm}P{\widw cm}P{\widw cm}P{\widw cm}}
	\includegraphics[width=0.1\textwidth]{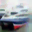}
	\includegraphics[width=0.1\textwidth]{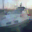}
	\includegraphics[width=0.1\textwidth]{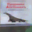}
	\includegraphics[width=0.1\textwidth]{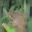}
	\includegraphics[width=0.1\textwidth]{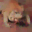}
	\includegraphics[width=0.1\textwidth]{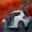}
	\includegraphics[width=0.1\textwidth]{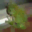}
	\includegraphics[width=0.1\textwidth]{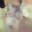}
	\includegraphics[width=0.1\textwidth]{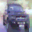}
	\includegraphics[width=0.1\textwidth]{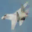}\\
	\includegraphics[width=0.1\textwidth]{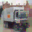}
	\includegraphics[width=0.1\textwidth]{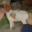}
	\includegraphics[width=0.1\textwidth]{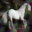}
	\includegraphics[width=0.1\textwidth]{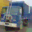}
	\includegraphics[width=0.1\textwidth]{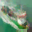}
	\includegraphics[width=0.1\textwidth]{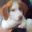}
	\includegraphics[width=0.1\textwidth]{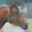}
	\includegraphics[width=0.1\textwidth]{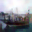}
	\includegraphics[width=0.1\textwidth]{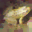}
	\includegraphics[width=0.1\textwidth]{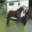}\\
	\includegraphics[width=0.1\textwidth]{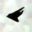}
	\includegraphics[width=0.1\textwidth]{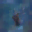}
	\includegraphics[width=0.1\textwidth]{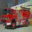}
	\includegraphics[width=0.1\textwidth]{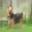}
	\includegraphics[width=0.1\textwidth]{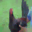}
	\includegraphics[width=0.1\textwidth]{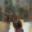}
	\includegraphics[width=0.1\textwidth]{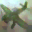}
	\includegraphics[width=0.1\textwidth]{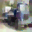}
	\includegraphics[width=0.1\textwidth]{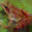}
	\includegraphics[width=0.1\textwidth]{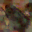}\\
	\includegraphics[width=0.1\textwidth]{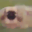}
	\includegraphics[width=0.1\textwidth]{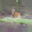}
	\includegraphics[width=0.1\textwidth]{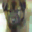}
	\includegraphics[width=0.1\textwidth]{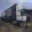}
	\includegraphics[width=0.1\textwidth]{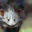}
	\includegraphics[width=0.1\textwidth]{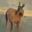}
	\includegraphics[width=0.1\textwidth]{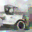}
	\includegraphics[width=0.1\textwidth]{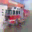}
	\includegraphics[width=0.1\textwidth]{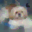}
	\includegraphics[width=0.1\textwidth]{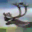}\\
	\includegraphics[width=0.1\textwidth]{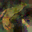}
	\includegraphics[width=0.1\textwidth]{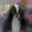}
	\includegraphics[width=0.1\textwidth]{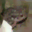}
	\includegraphics[width=0.1\textwidth]{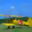}
	\includegraphics[width=0.1\textwidth]{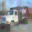}
	\includegraphics[width=0.1\textwidth]{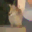}
	\includegraphics[width=0.1\textwidth]{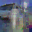}
	\includegraphics[width=0.1\textwidth]{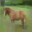}
	\includegraphics[width=0.1\textwidth]{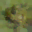}
	\includegraphics[width=0.1\textwidth]{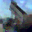}\\
	\includegraphics[width=0.1\textwidth]{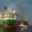}
	\includegraphics[width=0.1\textwidth]{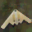}
	\includegraphics[width=0.1\textwidth]{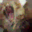}
	\includegraphics[width=0.1\textwidth]{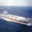}
	\includegraphics[width=0.1\textwidth]{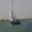}
	\includegraphics[width=0.1\textwidth]{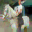}
	\includegraphics[width=0.1\textwidth]{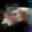}
	\includegraphics[width=0.1\textwidth]{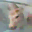}
	\includegraphics[width=0.1\textwidth]{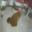}
	\includegraphics[width=0.1\textwidth]{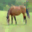}\\
	\includegraphics[width=0.1\textwidth]{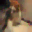}
	\includegraphics[width=0.1\textwidth]{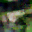}
	\includegraphics[width=0.1\textwidth]{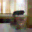}
	\includegraphics[width=0.1\textwidth]{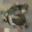}
	\includegraphics[width=0.1\textwidth]{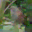}
	\includegraphics[width=0.1\textwidth]{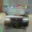}
	\includegraphics[width=0.1\textwidth]{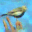}
	\includegraphics[width=0.1\textwidth]{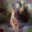}
	\includegraphics[width=0.1\textwidth]{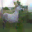}
	\includegraphics[width=0.1\textwidth]{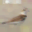}\\
	\includegraphics[width=0.1\textwidth]{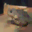}
	\includegraphics[width=0.1\textwidth]{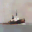}
	\includegraphics[width=0.1\textwidth]{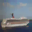}
	\includegraphics[width=0.1\textwidth]{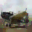}
	\includegraphics[width=0.1\textwidth]{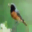}
	\includegraphics[width=0.1\textwidth]{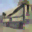}
	\includegraphics[width=0.1\textwidth]{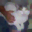}
	\includegraphics[width=0.1\textwidth]{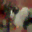}
	\includegraphics[width=0.1\textwidth]{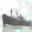}
	\includegraphics[width=0.1\textwidth]{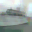}\\
	\includegraphics[width=0.1\textwidth]{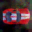}
	\includegraphics[width=0.1\textwidth]{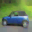}
	\includegraphics[width=0.1\textwidth]{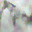}
	\includegraphics[width=0.1\textwidth]{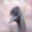}
	\includegraphics[width=0.1\textwidth]{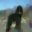}
	\includegraphics[width=0.1\textwidth]{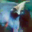}
	\includegraphics[width=0.1\textwidth]{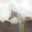}
	\includegraphics[width=0.1\textwidth]{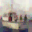}
	\includegraphics[width=0.1\textwidth]{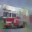}
	\includegraphics[width=0.1\textwidth]{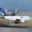}\\
	\includegraphics[width=0.1\textwidth]{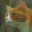}
	\includegraphics[width=0.1\textwidth]{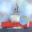}
	\includegraphics[width=0.1\textwidth]{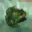}
	\includegraphics[width=0.1\textwidth]{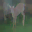}
	\includegraphics[width=0.1\textwidth]{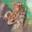}
	\includegraphics[width=0.1\textwidth]{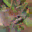}
	\includegraphics[width=0.1\textwidth]{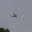}
	\includegraphics[width=0.1\textwidth]{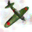}
	\includegraphics[width=0.1\textwidth]{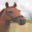}
	\includegraphics[width=0.1\textwidth]{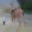}
\end{tabular}
    \caption{Results of the first $100$ experiments for $E=5$, $n=1$, $B=1$.}
    \label{fig:fedavg4}
\end{figure}

\begin{figure}
    \centering
    \begin{tabular}{l P{\widw cm}P{\widw cm}P{\widw cm}P{\widw cm}P{\widw cm}P{\widw cm}P{\widw cm}}
\includegraphics[width=0.11\textwidth]{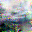}
\includegraphics[width=0.11\textwidth]{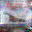}
\includegraphics[width=0.11\textwidth]{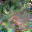}
\includegraphics[width=0.11\textwidth]{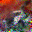}
\includegraphics[width=0.11\textwidth]{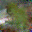}
\includegraphics[width=0.11\textwidth]{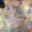}
\includegraphics[width=0.11\textwidth]{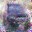}
\includegraphics[width=0.11\textwidth]{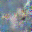}\\
\includegraphics[width=0.11\textwidth]{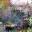}
\includegraphics[width=0.11\textwidth]{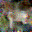}
\includegraphics[width=0.11\textwidth]{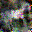}
\includegraphics[width=0.11\textwidth]{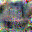}
\includegraphics[width=0.11\textwidth]{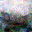}
\includegraphics[width=0.11\textwidth]{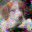}
\includegraphics[width=0.11\textwidth]{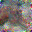}
\includegraphics[width=0.11\textwidth]{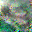}\\
\includegraphics[width=0.11\textwidth]{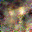}
\includegraphics[width=0.11\textwidth]{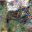}
\includegraphics[width=0.11\textwidth]{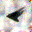}
\includegraphics[width=0.11\textwidth]{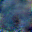}
\includegraphics[width=0.11\textwidth]{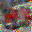}
\includegraphics[width=0.11\textwidth]{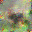}
\includegraphics[width=0.11\textwidth]{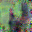}
\includegraphics[width=0.11\textwidth]{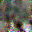}\\
\includegraphics[width=0.11\textwidth]{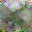}
\includegraphics[width=0.11\textwidth]{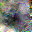}
\includegraphics[width=0.11\textwidth]{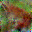}
\includegraphics[width=0.11\textwidth]{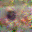}
\includegraphics[width=0.11\textwidth]{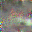}
\includegraphics[width=0.11\textwidth]{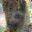}
\includegraphics[width=0.11\textwidth]{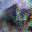}
\includegraphics[width=0.11\textwidth]{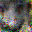}\\
\includegraphics[width=0.11\textwidth]{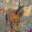}
\includegraphics[width=0.11\textwidth]{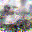}
\includegraphics[width=0.11\textwidth]{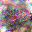}
\includegraphics[width=0.11\textwidth]{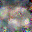}
\includegraphics[width=0.11\textwidth]{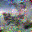}
\includegraphics[width=0.11\textwidth]{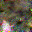}
\includegraphics[width=0.11\textwidth]{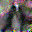}
\includegraphics[width=0.11\textwidth]{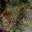}
\end{tabular}
    \caption{Results of the first ten experiments for $E=1$, $n=4$, $B=2$.}
    \label{fig:fedavg1}
\end{figure}

\begin{figure}
    \centering
\centering
\begin{tabular}{l P{\widw cm}P{\widw cm}P{\widw cm}P{\widw cm}P{\widw cm}P{\widw cm}P{\widw cm}}
\includegraphics[width=0.11\textwidth]{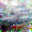}
\includegraphics[width=0.11\textwidth]{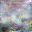}
\includegraphics[width=0.11\textwidth]{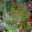}
\includegraphics[width=0.11\textwidth]{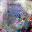}
\includegraphics[width=0.11\textwidth]{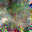}
\includegraphics[width=0.11\textwidth]{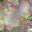}
\includegraphics[width=0.11\textwidth]{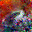}
\includegraphics[width=0.11\textwidth]{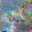}\\
\includegraphics[width=0.11\textwidth]{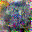}
\includegraphics[width=0.11\textwidth]{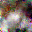}
\includegraphics[width=0.11\textwidth]{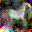}
\includegraphics[width=0.11\textwidth]{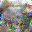}
\includegraphics[width=0.11\textwidth]{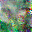}
\includegraphics[width=0.11\textwidth]{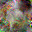}
\includegraphics[width=0.11\textwidth]{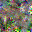}
\includegraphics[width=0.11\textwidth]{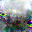}\\
\includegraphics[width=0.11\textwidth]{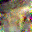}
\includegraphics[width=0.11\textwidth]{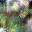}
\includegraphics[width=0.11\textwidth]{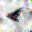}
\includegraphics[width=0.11\textwidth]{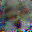}
\includegraphics[width=0.11\textwidth]{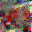}
\includegraphics[width=0.11\textwidth]{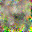}
\includegraphics[width=0.11\textwidth]{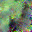}
\includegraphics[width=0.11\textwidth]{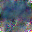}\\
\includegraphics[width=0.11\textwidth]{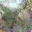}
\includegraphics[width=0.11\textwidth]{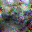}
\includegraphics[width=0.11\textwidth]{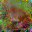}
\includegraphics[width=0.11\textwidth]{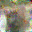}
\includegraphics[width=0.11\textwidth]{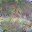}
\includegraphics[width=0.11\textwidth]{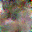}
\includegraphics[width=0.11\textwidth]{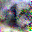}
\includegraphics[width=0.11\textwidth]{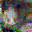}\\
\includegraphics[width=0.11\textwidth]{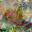}
\includegraphics[width=0.11\textwidth]{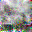}
\includegraphics[width=0.11\textwidth]{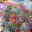}
\includegraphics[width=0.11\textwidth]{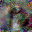}
\includegraphics[width=0.11\textwidth]{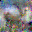}
\includegraphics[width=0.11\textwidth]{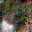}
\includegraphics[width=0.11\textwidth]{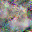}
\includegraphics[width=0.11\textwidth]{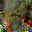}\\
\includegraphics[width=0.11\textwidth]{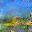}
\includegraphics[width=0.11\textwidth]{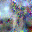}
\includegraphics[width=0.11\textwidth]{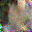}
\includegraphics[width=0.11\textwidth]{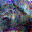}
\includegraphics[width=0.11\textwidth]{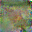}
\includegraphics[width=0.11\textwidth]{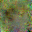}
\includegraphics[width=0.11\textwidth]{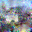}
\includegraphics[width=0.11\textwidth]{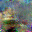}\\
\includegraphics[width=0.11\textwidth]{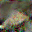}
\includegraphics[width=0.11\textwidth]{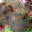}
\includegraphics[width=0.11\textwidth]{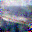}
\includegraphics[width=0.11\textwidth]{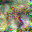}
\includegraphics[width=0.11\textwidth]{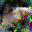}
\includegraphics[width=0.11\textwidth]{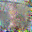}
\includegraphics[width=0.11\textwidth]{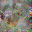}
\includegraphics[width=0.11\textwidth]{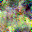}\\
\includegraphics[width=0.11\textwidth]{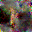}
\includegraphics[width=0.11\textwidth]{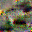}
\includegraphics[width=0.11\textwidth]{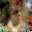}
\includegraphics[width=0.11\textwidth]{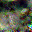}
\includegraphics[width=0.11\textwidth]{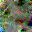}
\includegraphics[width=0.11\textwidth]{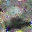}
\includegraphics[width=0.11\textwidth]{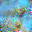}
\includegraphics[width=0.11\textwidth]{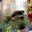}\\
\includegraphics[width=0.11\textwidth]{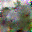}
\includegraphics[width=0.11\textwidth]{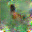}
\includegraphics[width=0.11\textwidth]{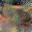}
\includegraphics[width=0.11\textwidth]{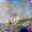}
\includegraphics[width=0.11\textwidth]{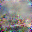}
\includegraphics[width=0.11\textwidth]{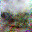}
\includegraphics[width=0.11\textwidth]{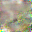}
\includegraphics[width=0.11\textwidth]{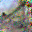}\\
\includegraphics[width=0.11\textwidth]{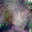}
\includegraphics[width=0.11\textwidth]{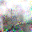}
\includegraphics[width=0.11\textwidth]{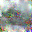}
\includegraphics[width=0.11\textwidth]{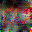}
\includegraphics[width=0.11\textwidth]{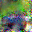}
\includegraphics[width=0.11\textwidth]{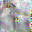}
\includegraphics[width=0.11\textwidth]{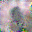}
\includegraphics[width=0.11\textwidth]{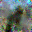}
\end{tabular}
    \caption{Results of the first ten experiments for $E=1$, $n=8$, $B=2$.}
    \label{fig:fedavg2}
\end{figure}

\begin{figure}
    \centering
    \begin{tabular}{l P{\widw cm}P{\widw cm}P{\widw cm}P{\widw cm}P{\widw cm}P{\widw cm}P{\widw cm}}
\includegraphics[width=0.11\textwidth]{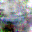}
\includegraphics[width=0.11\textwidth]{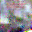}
\includegraphics[width=0.11\textwidth]{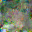}
\includegraphics[width=0.11\textwidth]{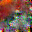}
\includegraphics[width=0.11\textwidth]{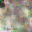}
\includegraphics[width=0.11\textwidth]{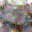}
\includegraphics[width=0.11\textwidth]{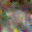}
\includegraphics[width=0.11\textwidth]{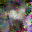}\\
\includegraphics[width=0.11\textwidth]{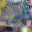}
\includegraphics[width=0.11\textwidth]{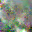}
\includegraphics[width=0.11\textwidth]{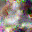}
\includegraphics[width=0.11\textwidth]{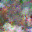}
\includegraphics[width=0.11\textwidth]{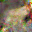}
\includegraphics[width=0.11\textwidth]{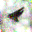}
\includegraphics[width=0.11\textwidth]{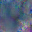}
\includegraphics[width=0.11\textwidth]{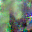}\\
\includegraphics[width=0.11\textwidth]{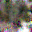}
\includegraphics[width=0.11\textwidth]{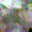}
\includegraphics[width=0.11\textwidth]{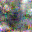}
\includegraphics[width=0.11\textwidth]{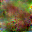}
\includegraphics[width=0.11\textwidth]{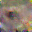}
\includegraphics[width=0.11\textwidth]{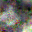}
\includegraphics[width=0.11\textwidth]{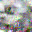}
\includegraphics[width=0.11\textwidth]{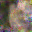}\\
\includegraphics[width=0.11\textwidth]{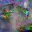}
\includegraphics[width=0.11\textwidth]{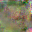}
\includegraphics[width=0.11\textwidth]{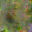}
\includegraphics[width=0.11\textwidth]{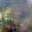}
\includegraphics[width=0.11\textwidth]{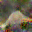}
\includegraphics[width=0.11\textwidth]{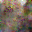}
\includegraphics[width=0.11\textwidth]{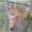}
\includegraphics[width=0.11\textwidth]{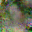}\\
\includegraphics[width=0.11\textwidth]{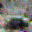}
\includegraphics[width=0.11\textwidth]{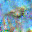}
\includegraphics[width=0.11\textwidth]{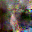}
\includegraphics[width=0.11\textwidth]{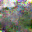}
\includegraphics[width=0.11\textwidth]{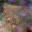}
\includegraphics[width=0.11\textwidth]{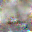}
\includegraphics[width=0.11\textwidth]{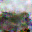}
\includegraphics[width=0.11\textwidth]{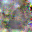}\\
\includegraphics[width=0.11\textwidth]{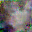}
\includegraphics[width=0.11\textwidth]{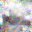}
\includegraphics[width=0.11\textwidth]{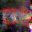}
\includegraphics[width=0.11\textwidth]{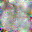}
\includegraphics[width=0.11\textwidth]{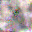}
\includegraphics[width=0.11\textwidth]{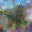}
\includegraphics[width=0.11\textwidth]{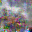}
\includegraphics[width=0.11\textwidth]{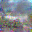}\\
\includegraphics[width=0.11\textwidth]{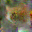}
\includegraphics[width=0.11\textwidth]{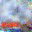}
\includegraphics[width=0.11\textwidth]{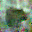}
\includegraphics[width=0.11\textwidth]{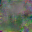}
\includegraphics[width=0.11\textwidth]{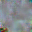}
\includegraphics[width=0.11\textwidth]{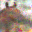}
\includegraphics[width=0.11\textwidth]{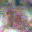}
\includegraphics[width=0.11\textwidth]{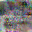}\\
\includegraphics[width=0.11\textwidth]{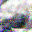}
\includegraphics[width=0.11\textwidth]{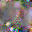}
\includegraphics[width=0.11\textwidth]{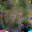}
\includegraphics[width=0.11\textwidth]{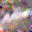}
\includegraphics[width=0.11\textwidth]{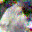}
\includegraphics[width=0.11\textwidth]{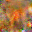}
\includegraphics[width=0.11\textwidth]{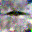}
\includegraphics[width=0.11\textwidth]{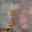}\\
\includegraphics[width=0.11\textwidth]{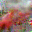}
\includegraphics[width=0.11\textwidth]{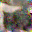}
\includegraphics[width=0.11\textwidth]{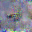}
\includegraphics[width=0.11\textwidth]{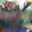}
\includegraphics[width=0.11\textwidth]{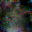}
\includegraphics[width=0.11\textwidth]{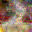}
\includegraphics[width=0.11\textwidth]{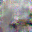}
\includegraphics[width=0.11\textwidth]{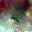}\\
\includegraphics[width=0.11\textwidth]{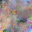}
\includegraphics[width=0.11\textwidth]{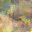}
\includegraphics[width=0.11\textwidth]{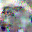}
\includegraphics[width=0.11\textwidth]{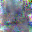}
\includegraphics[width=0.11\textwidth]{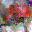}
\includegraphics[width=0.11\textwidth]{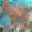}
\includegraphics[width=0.11\textwidth]{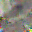}
\includegraphics[width=0.11\textwidth]{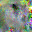}
\end{tabular}
    \caption{Results of the first ten experiments for $E=1$, $n=8$, $B=8$.}
    \label{fig:fedavg3}
\end{figure}

\begin{figure}
    \centering
    \begin{tabular}{l P{\widw cm}P{\widw cm}P{\widw cm}P{\widw cm}P{\widw cm}P{\widw cm}P{\widw cm}}
\includegraphics[width=0.11\textwidth]{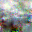}
\includegraphics[width=0.11\textwidth]{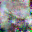}
\includegraphics[width=0.11\textwidth]{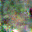}
\includegraphics[width=0.11\textwidth]{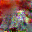}
\includegraphics[width=0.11\textwidth]{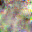}
\includegraphics[width=0.11\textwidth]{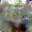}
\includegraphics[width=0.11\textwidth]{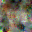}
\includegraphics[width=0.11\textwidth]{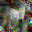}\\
\includegraphics[width=0.11\textwidth]{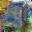}
\includegraphics[width=0.11\textwidth]{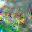}
\includegraphics[width=0.11\textwidth]{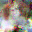}
\includegraphics[width=0.11\textwidth]{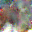}
\includegraphics[width=0.11\textwidth]{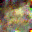}
\includegraphics[width=0.11\textwidth]{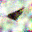}
\includegraphics[width=0.11\textwidth]{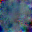}
\includegraphics[width=0.11\textwidth]{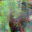}\\
\includegraphics[width=0.11\textwidth]{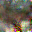}
\includegraphics[width=0.11\textwidth]{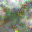}
\includegraphics[width=0.11\textwidth]{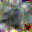}
\includegraphics[width=0.11\textwidth]{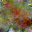}
\includegraphics[width=0.11\textwidth]{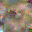}
\includegraphics[width=0.11\textwidth]{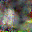}
\includegraphics[width=0.11\textwidth]{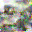}
\includegraphics[width=0.11\textwidth]{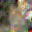}\\
\includegraphics[width=0.11\textwidth]{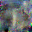}
\includegraphics[width=0.11\textwidth]{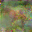}
\includegraphics[width=0.11\textwidth]{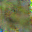}
\includegraphics[width=0.11\textwidth]{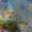}
\includegraphics[width=0.11\textwidth]{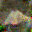}
\includegraphics[width=0.11\textwidth]{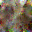}
\includegraphics[width=0.11\textwidth]{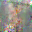}
\includegraphics[width=0.11\textwidth]{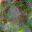}\\
\includegraphics[width=0.11\textwidth]{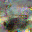}
\includegraphics[width=0.11\textwidth]{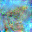}
\includegraphics[width=0.11\textwidth]{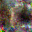}
\includegraphics[width=0.11\textwidth]{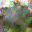}
\includegraphics[width=0.11\textwidth]{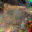}
\includegraphics[width=0.11\textwidth]{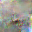}
\includegraphics[width=0.11\textwidth]{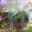}
\includegraphics[width=0.11\textwidth]{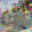}\\
\includegraphics[width=0.11\textwidth]{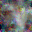}
\includegraphics[width=0.11\textwidth]{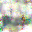}
\includegraphics[width=0.11\textwidth]{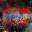}
\includegraphics[width=0.11\textwidth]{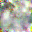}
\includegraphics[width=0.11\textwidth]{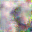}
\includegraphics[width=0.11\textwidth]{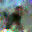}
\includegraphics[width=0.11\textwidth]{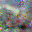}
\includegraphics[width=0.11\textwidth]{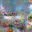}\\
\includegraphics[width=0.11\textwidth]{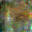}
\includegraphics[width=0.11\textwidth]{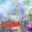}
\includegraphics[width=0.11\textwidth]{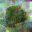}
\includegraphics[width=0.11\textwidth]{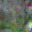}
\includegraphics[width=0.11\textwidth]{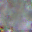}
\includegraphics[width=0.11\textwidth]{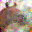}
\includegraphics[width=0.11\textwidth]{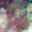}
\includegraphics[width=0.11\textwidth]{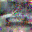}\\
\includegraphics[width=0.11\textwidth]{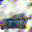}
\includegraphics[width=0.11\textwidth]{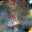}
\includegraphics[width=0.11\textwidth]{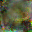}
\includegraphics[width=0.11\textwidth]{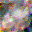}
\includegraphics[width=0.11\textwidth]{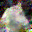}
\includegraphics[width=0.11\textwidth]{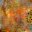}
\includegraphics[width=0.11\textwidth]{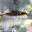}
\includegraphics[width=0.11\textwidth]{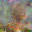}\\
\includegraphics[width=0.11\textwidth]{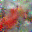}
\includegraphics[width=0.11\textwidth]{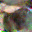}
\includegraphics[width=0.11\textwidth]{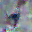}
\includegraphics[width=0.11\textwidth]{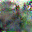}
\includegraphics[width=0.11\textwidth]{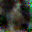}
\includegraphics[width=0.11\textwidth]{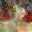}
\includegraphics[width=0.11\textwidth]{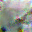}
\includegraphics[width=0.11\textwidth]{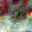}\\
\includegraphics[width=0.11\textwidth]{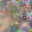}
\includegraphics[width=0.11\textwidth]{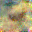}
\includegraphics[width=0.11\textwidth]{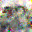}
\includegraphics[width=0.11\textwidth]{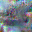}
\includegraphics[width=0.11\textwidth]{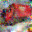}
\includegraphics[width=0.11\textwidth]{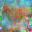}
\includegraphics[width=0.11\textwidth]{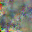}
\includegraphics[width=0.11\textwidth]{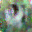}
\end{tabular}
    \caption{Results of the first ten experiments for $E=5$, $n=8$, $B=8$.}
    \label{fig:fedavg5}
\end{figure}

\phantom{asdasdddddddddddddddddddddddddddddddddddddddddddddddddddddddddddddddddddddddddddddddddddddddddddddddddddddddddddddddddddddddddddddddddddddddddddddddddddddddddddddddddddddddddddddddddddddddddddddddddddddddddddddddddddddddddddddddddddddddddddddddddddddddddddddddddddddddddddddddddddddddddddddddddddddddddddddddddddddddddddddddddddddddddddddddddddddddddddddddddddddddddddddddddddddddddddddddddddddddddddddddddddddddddd} 

\noindent Additional images are following on the next pages.

\begin{figure}
    \centering
    \includegraphics[height=0.44\textheight]{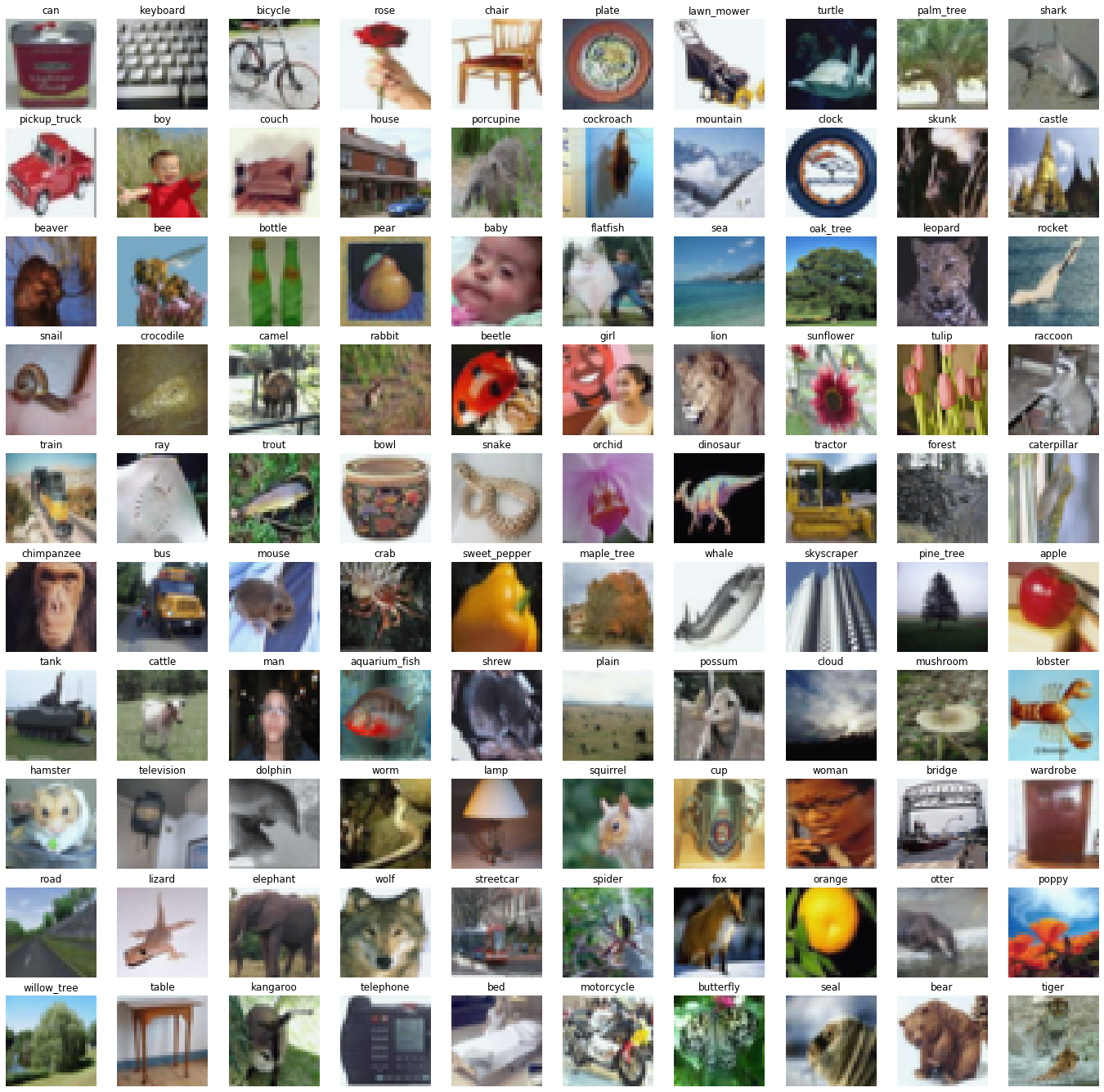}
    \includegraphics[height=0.44\textheight]{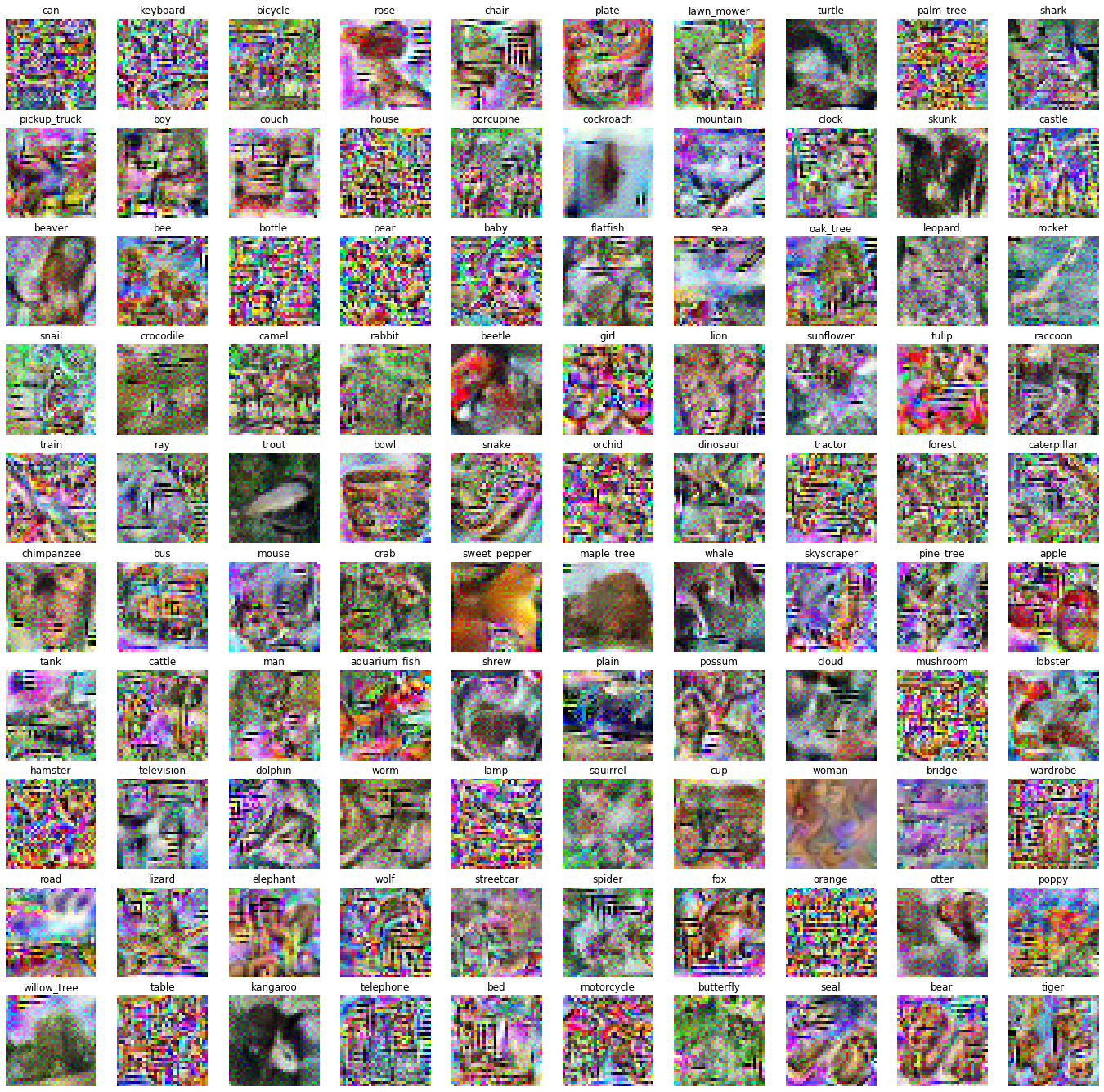}
    \caption{Full results for the batch of CIFAR-100 images. Same experiment as in Fig. 6 of the paper.}
    \label{fig:full_cifar100_2}
\end{figure}
\end{document}